\def\alt{0}

\if\alt0
\documentclass{article}
\usepackage[margin=1in]{geometry}
\usepackage{amsthm}
\usepackage[linesnumbered,ruled,vlined]{algorithm2e}

\else
\documentclass[anon,12pt]{alt2021}
\fi

\usepackage{ifthen}
\usepackage{graphicx}
\usepackage{setspace}
\usepackage{caption}
\usepackage{subcaption}
\usepackage{amsmath}
\usepackage{listings}
\usepackage{indentfirst}
\usepackage{algpseudocode}

\usepackage{amssymb}
\usepackage{mathtools}
\usepackage{floatrow}
\usepackage{url}
\usepackage{hyperref}
\hypersetup{
    colorlinks=true,
    linkcolor=blue,
    filecolor=magenta,      
    urlcolor=cyan,
    citecolor=blue,
}
\if\alt0

\newtheorem*{rep@theorem}{\rep@title}
\newcommand{\newreptheorem}[2]{%
\newenvironment{rep#1}[1]{%
 \def\rep@title{#2 \ref*{##1}}%
 \begin{rep@theorem}}%
 {\end{rep@theorem}}}
\makeatother

\newtheorem{theorem}{Theorem}[section]
\newreptheorem{theorem}{Theorem}
\newtheorem{corollary}[theorem]{Corollary}
\newtheorem{lemma}[theorem]{Lemma}
\newreptheorem{lemma}{Lemma}
\newtheorem{proposition}[theorem]{Proposition}
\newtheorem{definition}[theorem]{Definition}
\newtheorem{remark}[theorem]{Remark}
\usepackage[dvipsnames]{xcolor}
\fi

\DeclareMathOperator{\TV}{TV}
\newcommand{\PD}{{\mathbb{S}_{d}}}
\newcommand{\ball}[3]{\mathcal{B}\left( #1, #2, #3 \right)}
\newcommand{\eps}{\varepsilon}

\if\alt0
\title{On the Sample Complexity of Privately Learning \\ Unbounded High-Dimensional Gaussians\footnote{Authors are listed in alphabetical order.}}

\author{
  Ishaq Aden-Ali\thanks{Department of Computing and Software, McMaster University. Supported by an Ontario Graduate Scholarship. \texttt{adenali@mcmaster.ca}.}
  \and
  Hassan Ashtiani\thanks{Department of Computing and Software, McMaster University. \texttt{zokaeiam@mcmaster.ca}. Supported by an NSERC Discovery grant and a McMaster University startup grant.}
  \and
  Gautam Kamath\thanks{Cheriton School of Computer Science, University of Waterloo.  \texttt{g@csail.mit.edu}. Supported by an NSERC
Discovery grant, a Compute Canada RRG grant, and a University of Waterloo startup grant.}
}
\date{October 19, 2020}

\else
\title[Privately Learning Unbounded Gaussians]{On the Sample Complexity of Privately Learning Unbounded High-Dimensional Gaussians}
\usepackage{times}

\altauthor{%
 \coltauthor{\Name{Ishaq Aden-Ali} \Email{adenali@mcmaster.ca}\and
 \Name{Hassan Ashtiani} \Email{zokaeiam@mcmaster.ca}\\
 \addr Department of Computing and Software, McMaster University}
 \AND
 \Name{Gautam Kamath} \Email{g@csail.mit.edu}\\
 \addr Cheriton School of Computer Science, University of Waterloo.%
 
}
\fi

\begin{document}

\maketitle

    
        

    
    


\begin{abstract}
    We provide sample complexity upper bounds for agnostically learning multivariate Gaussians under the constraint of approximate differential privacy.
    These are the first finite sample upper bounds for general Gaussians which do not impose restrictions on the parameters of the distribution. Our bounds are near-optimal in the case when the covariance is known to be the identity, and conjectured to be near-optimal in the general case.
    From a technical standpoint, we provide analytic tools for arguing the existence of global ``locally small'' covers from local covers of the space.
    These are exploited using modifications of recent techniques for differentially private hypothesis selection.
    Our techniques may prove useful for privately learning other distribution classes which do not possess a finite cover.
\end{abstract}
\section{Introduction}
Given samples from a distribution $P$, can we estimate the underlying distribution?
This problem has a long and rich history, culminating in a mature understanding for many settings of interest.
However, in many cases the dataset may consist of sensitive data belonging to individuals, and naive execution of classic methods may inadvertently result in private information leakage.
See, for instance, privacy attacks described in such estimation settings including~\cite{DinurN03,HomerSRDTMPSNC08,BunUV14, DworkSSUV15, ShokriSSS17}, and the survey~\cite{DworkSSU17}. 

To address concerns of this nature, in 2006, Dwork, McSherry, Nissim, and Smith introduced the celebrated notion of differential privacy (DP)~\cite{DworkMNS06}, which provides a strong standard for data privacy.
It ensures that no single data point has significant influence on the output of the algorithm, thus masking the contribution of individuals in the dataset.
Differential privacy has seen practical adoption in many organizations, including Apple~\cite{AppleDP17}, Google~\cite{ErlingssonPK14,BittauEMMRLRKTS17}, Microsoft~\cite{DingKY17}, and the US Census Bureau~\cite{DajaniLSKRMGDGKKLSSVA17}.
At this point, there is a rich body of literature, giving differentially private algorithms for a wide array of tasks.

There has recently been significant interest in distribution and parameter estimation under differential privacy (see Section~\ref{sec:related} for discussion of related work). 
Most relevant to our investigation is the work of Bun, Kamath, Steinke, and Wu~\cite{BunKSW19}, which provides a generic framework that, given a cover for a class of distributions, describes a private algorithm for learning said class with sample complexity logarithmic in the size of the cover.
There, the privacy guarantee is the strongest notion of \emph{pure} $(\varepsilon, 0)$-differential privacy.

An obvious drawback of this approach is that it fails to provide sample complexity upper bounds for estimating classes of distributions which do not possess a finite cover.
The canonical example is the set of all Gaussian distributions.
It turns out that this is inherently impossible -- ``packing lower bounds'' imply that no finite sample algorithm exists for such cases under pure differential privacy.
This theoretical limitation can have significant practical implications as well, as it forces the data analyst to choose between having good accuracy and preserving privacy. 
It turns out that, under pure differential privacy, the only way to avoid this issue is to assume the underlying distribution belongs to a more restricted class -- such as Gaussian distributions with bounded mean and covariance. 


On the other hand, stronger results are possible if one relaxes the privacy notion to the weaker guarantee of approximate differential privacy~\cite{DworkKMMN06}.
In particular, it is known that this relaxation permits ``stability-based'' approaches, which can avoid issues associated with infinite covers by pinpointing the area where ``a lot of the data lies,'' see, e.g., the classic example of the stability-based histogram~\cite{KorolovaKMN09, BunNS16}.

Using such approaches, \cite{BunKSW19} exploit the stability-based GAP-MAX algorithm of~\cite{BunDRS18}, refining their framework to provide algorithms under approximate differential privacy.
The caveat is that this time the approach requires construction of a more sophisticated object: a cover for the class which is in a certain technical sense ``locally small.''
Such locally small covers are much more difficult to construct and analyze. As such, \cite{BunKSW19} only provide them for univariate Gaussian and multivariate Gaussians with identity covariance.
As the most notable omission, they do not provide a locally small cover for general multivariate Gaussians.
Indeed, for Gaussians with identity covariance, it is easy to reason about the local size of covers, as total variation distance between distributions corresponds to the $\ell_2$-distance between their means.
However, when the covariance is not fixed, the total variation distance is characterized by the \emph{Mahalanobis distance}, which has a significantly more sophisticated geometry.
Analyzing these situations to show local smallness appears to be intractable using current analytic techniques, which involve explicitly constructing and analyzing a cover of the space.
Given this challenge, up to now it has not been clear even whether a finite sample algorithm exists at all!
And this is only for the fundamental case of Gaussians, raising the question of how one would even approach more complex classes of distributions.

\subsection{Results and Techniques}
We resolve these issues by providing a simpler method for proving existence of locally small covers.
These lead to our main results, sample complexity upper bounds for semi-agnostically learning Gaussian distributions under approximate differential privacy.
\begin{theorem}[Informal version of Theorem~\ref{thm:Gaussian-sample complexity-agnostic}]
\label{thm:informal-general}
The sample complexity of semi-agnostically learning a $d$-dimensional Gaussian distribution to $\alpha$-accuracy in total variation distance under $(\varepsilon,\delta)$-differential privacy is
\[\tilde O\left(\frac{d^2}{\alpha^2} + \frac{d^2}{\alpha\varepsilon} + \frac{\log(1/\delta)}{\varepsilon} \right). \]
\end{theorem}
This is the first sample complexity bound for privately learning a multivariate Gaussian with no conditions on the covariance matrix.
The first and third terms are known to be tight, and there is strong evidence that the second is as well, see Section~\ref{sec:lower-bounds}.
The previous best algorithm was that of~\cite{KamathLSU19}, which provided the stronger guarantee of concentrated differential privacy~\cite{DworkR16,BunS16} (which is intermediate to pure and approximate DP).
However, it required the true covariance $\Sigma$ to be bounded as $I \preceq \Sigma \preceq KI$ for some known parameter $K$, and the third term in the sample complexity is instead $O\left(\frac{d^{3/2}\log^{1/2} K}{\varepsilon}\right)$, which is prohibitive for large (or unknown) $K$.
In contrast, our result holds for unrestricted Gaussian distributions.

We also provide a better upper bound for the case when the covariance matrix is known.


\begin{theorem}[Informal version of Theorem~\ref{thm:Location-gaussian-sample-complexity-agnostic}]
\label{thm:informal-identity}
The sample complexity of semi-agnostically learning a $d$-dimensional Gaussian distribution with known covariance to $\alpha$-accuracy in total variation distance under $(\varepsilon,\delta)$-differential privacy is
\[\tilde O\left(\frac{d}{\alpha^2} + \frac{d}{\alpha\varepsilon} + \frac{\log(1/\delta)}{\varepsilon} \right). \]
\end{theorem}
This is the first bound which achieves a near-optimal dependence simultaneously on all parameters, see Section~\ref{sec:lower-bounds}.
In particular, it improves upon previous results in which the third term is replaced by $O\left(\frac{\log(1/\delta)}{\alpha\varepsilon}\right)$~\cite{BunKSW19} or $O\left(\frac{\sqrt{d}\log^{3/2}(1/\delta)}{\varepsilon}\right)$~\cite{KarwaV18,KamathLSU19,BunKSW19}.

While we apply our approach to multivariate Gaussian estimation, it should more broadly apply to other classes of distributions with no finite-sized cover.

As mentioned before, we build upon the approach of Bun, Kamath, Steinke, and Wu~\cite{BunKSW19} to provide methods better suited for estimation under the constraint of approximate differential privacy.
Their work focuses primarily on pure DP distribution estimation for classes of distributions with a finite cover. 
Specifically, given a class of distributions with an $\alpha$-cover of size $\mathcal{C_\alpha}$, they give a pure DP algorithm for learning said class in total variation distance with sample complexity $O(\log |\mathcal{C_\alpha}|)$.
Naturally, this gives vacuous bounds for classes with an infinite cover -- indeed, packing lower bounds show that this is inherent under pure DP~\cite{HardtT10,BeimelBKN14, BunKSW19}.
To avoid these lower bounds, they show that learning is still possible if one relaxes to approximate DP and considers a ``locally small'' cover: one that has at most $k$ elements which are within an $O(\alpha)$-total variation distance ball of any element in the set.
The sample complexity of the resulting method does not depend on $|\mathcal{C}_\alpha|$, and instead we pay logarithmically in the parameter $k$.
They apply this framework to provide algorithms for estimating general univariate Gaussians, and multivariate Gaussians with identity covariance.
However, their arguments construct explicit covers for these cases, and it appears difficult to construct and analyze covers in situations with a rich geometric structure, such as multivariate Gaussians.
Indeed, it seems difficult in these settings to reason that a set is simultaneously a cover (i.e., every distribution in the class has a close element) and locally small (i.e., every distribution does not have \emph{too many} close elements).

We avoid this tension by taking a myopic view: in Lemma~\ref{lem:covering-balls-implies-locally-small-cover}, we show that if we can construct a cover with few elements for the neighbourhood of each \emph{individual} distribution, then there exists a locally small cover for the \emph{entire space}. 
This makes it significantly easier to reason about locally small covers, as we only have to consider covering a single distribution at a time, and we do not have to reason about how the elements that cover each distribution overlap with each other.
For example: to cover the neighbourhood of a single Gaussian with (full rank) covariance $\Sigma$, we can transform the covariance to the identity by multiplying by $\Sigma^{-1/2}$, cover the neighbourhood of $N(0,I)$ (which is easier), and transform the cover back to the original domain.
This is far simpler than trying to understand how to simultaneously cover multiple Gaussians with differently shaped covariance matrices in a locally small manner.
Our results for covering are presented in Section~\ref{sec:Covering}.

We then go on to apply these locally small covers to derive learning sample complexity upper bounds in Section~\ref{sec:learning}.
As mentioned before, this is done in~\cite{BunKSW19}, though we refine their method to achieve stronger bounds.
While this refinement is simple, we believe it to be important both technically (as it allows us to achieve likely near-optimal sample complexities) and conceptually (as we believe it clearly identifies what the ``hard part'' of the problem is). To elaborate, our approach can be divided into two steps;
\begin{enumerate}
    \item {\bf Coarse Estimation.} Find any distribution which is $0.99$-close to the true distribution, using the approximate DP GAP-MAX algorithm in~\cite{BunKSW19}.
    \item {\bf Fine Estimation.} Generate an $O(\alpha)$-cover around the distribution from the previous step, and run the pure DP private hypothesis selection algorithm in~\cite{BunKSW19}.
\end{enumerate}
We are not the first to use this type of two-step approach, as such decomposition has been previously applied, e.g.,~\cite{KarwaV18, KamathLSU19, KamathSU20}.
However, it was not applied in the context of the GAP-MAX algorithm in~\cite{BunKSW19}, preventing them from getting the right dependencies on all parameters -- in particular, it was not clear how to disentangle the dependencies on $\log(1/\delta)$ and $1/\alpha$ using their method directly.

Other beneficial features of this two-step approach, which have also been exploited in the past, include its modularity and the fact that the first and second steps involve qualitatively different privacy guarantees. 
However, we additionally comment how coarse an estimate required in the first step -- while the description above states that we require a $0.99$-close distribution, we may actually only need one with total variation distance bounded by $1 - \zeta$, where $\zeta$ may be exponentially small in the parameters of the problem!
See Remark~\ref{rem:coarse-cover}.
We hope that shining a light on this somewhat unconventional regime for private distribution estimation, which only requires a ``whiff'' of the true distribution, will inspire further investigation.

As a final contribution, in Section~\ref{sec:efficient-agnostic}, we revisit the generic private hypothesis selection problem. 
The main result of \cite{BunKSW19} is an algorithm for this problem which requires knowledge of the distance to the best hypothesis. 
They then wrap this algorithm in another procedure which ``guesses'' the distance to the best hypothesis, resulting in a semi-agnostic algorithm.
However, this loses large factors in the agnostic guarantee and is rather indirect.
We instead analyze the privatization of a different algorithm, the minimum distance estimate, which gives a semi-agnostic algorithm directly, with an optimal agnostic constant (i.e., providing a tight factor of 3~\cite{DevroyeL01}).
In our opinion, the algorithm and proof are even simpler than the non-agnostic algorithm of~\cite{BunKSW19}.

\begin{theorem}
[Informal version of Theorem~\ref{lem:private-mde}]
There exists an $\varepsilon$-differentially private algorithm for semi-agnostic hypothesis selection from a set of $m$ distributions $\mathcal{H} = \{H_1, \dots, H_m\}$.
The algorithm requires $n$ samples from some distribution $P$ and returns a distribution $\hat H \in \mathcal{H}$ where $\TV(P, \hat H) \leq 3 \cdot \mathrm{OPT} + \alpha$,
where 
\[n = O\left(\frac{\log m}{\alpha^2} + \frac{\log m}{\alpha \varepsilon}\right) .\]
\end{theorem}

\subsubsection{Comparison with Lower Bounds}
\label{sec:lower-bounds}
It is folklore that the non-private sample complexity of estimating a single $d$-dimensional Gaussian to accuracy $\alpha$ in total variation distance is $\Theta(d^2/\alpha^2)$, or, in the case when the covariance is the identity, $\Theta(d/\alpha^2)$.
Therefore the leading terms in the sample complexity bounds of Theorems~\ref{thm:informal-general} and~\ref{thm:informal-identity} are tight. 

Lower bounds for private statistical estimation are comparatively less explored.
Karwa and Vadhan~\cite{KarwaV18} showed a lower bound of $\Omega(\log(1/\delta)/\varepsilon)$, even for the simple case of estimating the mean of a univariate Gaussian with known variance, thus matching the third terms in Theorems~\ref{thm:informal-general} and~\ref{thm:informal-identity}.

However, approximate DP lower bounds in the multivariate setting are notoriously hard to come by, with the predominant technique being the ``fingerprinting'' approach~\cite{BunUV14, SteinkeU15, DworkSSUV15,KamathLSU19,CaiWZ19}.
Using this technique, Kamath, Li, Singhal, and Ullman~\cite{KamathLSU19} show a lower bound of $\tilde \Omega(d/\alpha\varepsilon)$ for Gaussian estimation with identity covariance, thus nearly-matching the second and final term in Theorem~\ref{thm:informal-identity}.
We note that this sample complexity does not change when we convert to the stronger notion of \emph{pure} differential privacy~\cite{KamathLSU19, BunKSW19}.

\cite{KamathLSU19} also proves a lower bound of $\Omega(d^2/\alpha\varepsilon)$ for general Gaussian estimation under \emph{pure} differential privacy.
Using the aforementioned invariance of the complexity of estimation with identity covariance under pure and approximate DP, we take this as strong evidence that there exists a lower bound of $\Omega(d^2/\alpha\varepsilon)$ for estimation of general Gaussians under approximate DP as well.

\subsubsection{Additional Related Work}
\label{sec:related}
The most relevant works are those on private distribution and parameter estimation, particularly in multivariate settings~\cite{NissimRS07, BunUV14,SteinkeU17a, SteinkeU17b, DworkSSUV15, BunSU17, KarwaV18, KamathLSU19, CaiWZ19, BunKSW19, AcharyaSZ20, KamathSU20, BiswasDKU20}.
While many of these focus on settings with parameters bounded by some constant, some pay particular attention to the cost in terms of this bound, including~\cite{KarwaV18, KamathLSU19, BunKSW19, BiswasDKU20, DuFMBG20}.
Private ball-finding algorithms~\cite{NissimSV16,NissimS18} provide approximate DP approaches for finding small $\ell_2$-balls containing many points, which can be applied in sample-and-aggregate settings~\cite{KamathSSU19}.
However, these too appear to be unable to exploit the sophisticated geometry that arises with non-spherical covariance matrices of general Gaussians.
Broadly speaking, we are not aware of any existing approach that is able to entirely avoid dependence on bounds on the parameters.
Other works on differentially private estimation include~\cite{DworkL09,Smith11, DiakonikolasHS15, BunS19, ZhangKKW20}. 
See~\cite{KamathU20} for more coverage of recent works in private statistics.

The work of Bun, Kamath, Steinke, and Wu~\cite{BunKSW19} is built upon classic results in hypothesis selection, combined with the exponential mechanism~\cite{McSherryT07}.
The underlying non-private approach was pioneered by Yatracos~\cite{Yatracos85}, and refined in subsequent work by Devroye and Lugosi~\cite{DevroyeL96, DevroyeL97, DevroyeL01}.
After this, additional considerations have been taken into account, such as computation, approximation factor, robustness, and more~\cite{MahalanabisS08, DaskalakisDS12b, DaskalakisK14, SureshOAJ14, AcharyaJOS14b, DiakonikolasKKLMS16, ashtiani2017sample, AshtianiBHLMP18, AcharyaFJOS18, BousquetKM19, aden2020sample}.
Notably, these primitives have also been translated to the more restrictive setting of \emph{local} differential privacy~\cite{GopiKKNWZ20}.
Similar techniques have also been exploited in a federated setting~\cite{LiuSYKR20}.

\section{Preliminaries}\label{sec:prelim}

\subsection{Notation}

For any $m \in \mathbb{N}$, $[m]$ denotes the set $\{1, 2,\dots, m\}$. Let $X \sim P$ denote a random variable $X$ sampled from distribution $P$. Let $(X_i)_{i=1}^m\sim P^m$ denote an i.i.d.\ random sample of size $m$ from distribution $P$. For a distribution $Q$ over a domain $\mathcal{X}$ and a set $A \subseteq \mathcal{X}$, we define $Q(A)$ as the probability $Q$ assigns to the event $A$.

For a positive integer $d$ let $\PD \subset \mathbb{R}^{d \times d}$ be the set of all $d$-by-$ d$ positive semi-definite real matrices. For a matrix $A \in \mathbb{R}^{m \times n}$, define $\| A \|_{1,1} = \sum_{i=1}^{m}\sum_{j=1}^{n} |A_{ij} | $ and $\| A \|_{\infty,\infty} = \max_{i,j} |A_{ij} | $. The determinant of a square matrix $A$ is given by $\text{det}(A)$.

\begin{definition}
A $d$-dimensional \emph{Gaussian distribution} $\mathcal{N}(\mu, \Sigma)$ with mean $\mu \in \mathbb{R}^{d}$ and covariance matrix $\Sigma \in \PD$ is a distribution with density function:
  $$p(x) = \frac{\exp\left(-\frac{1}{2}(x - \mu)^T\Sigma^{-1}(x - \mu)\right)}{\sqrt{(2\pi)^d\cdot\emph{det}(\Sigma)}}.$$
\end{definition}

We define the set of $d$-dimensional \textit{location Gaussians} as $\mathcal{G}^{L}_{d} := \left\{\mathcal{N}\left(\mu,I\right) : \mu \in \mathbb{R}^{d}\right\}$, and the set of $d$-dimensional \textit{scale Gaussians} as $\mathcal{G}^{S}_{d} := \left\{\mathcal{N}\left(0,\Sigma\right) : \Sigma \in \PD\right\}$. Define the set of (all) Gaussians as $\mathcal{G}_{d}:= \left\{\mathcal{N}\left(\mu,\Sigma\right) : \mu \in \mathbb{R}^{d},  \Sigma \in \PD\right\}$. A useful property of Gaussian distributions is that any linear transformation of a Gaussian random vector is also a Gaussian random vector. In particular, if $X \sim \mathcal{N}(\mu,\Sigma)$ is a $d$-dimensional Gaussian random vector and $A$ and $b$ are a $d$-dimensional square matrix and vector respectively, it follows that
\begin{equation}\label{eq:Gaussiantransform}
AX+b \sim \mathcal{N}(A\mu+b,A\Sigma A^{T}).
\end{equation}

\subsection{Distribution Learning}

A \emph{distribution learning method} is an algorithm that, given a sequence of i.i.d.\ samples from a distribution $P$, outputs a distribution $\widehat{H}$ as an estimate of $P$. The focus of this paper is on absolutely continuous probability distributions (distributions that have a density with respect to the Lebesgue measure), so we will refer to a probability distribution and its probability density function interchangeably. The specific measure of ``closeness'' between distributions that we use is the \emph{total variation distance}:
\begin{definition}
  Let $P$ and $Q$ be two probability distributions defined over $\mathcal{X}$ and let $\Omega$ be the Borel sigma-algebra on $\mathcal{X}$. The \emph{total variation distance} between $P$ and $Q$ is defined as 
  $$\TV(P,Q) = \sup_{S \in \Omega} |P(S)-Q(S)| =  \frac{1}{2} \int_{x \in \mathcal{X}} |P(x) - Q(x)|\mathrm{d}x = \frac{1}{2}\|P - Q\|_1 \in [0,1]. $$
  Moreover, if $\mathcal{H}$ is a set of distributions over a common domain, we define $\TV(P, \mathcal{H}) = \inf_{H\in \mathcal{H}} \TV(P, H)$.
\end{definition} 
Given $X \sim P$ and $Y \sim Q$, it is often useful for us to overload notation and define $\TV(X,Y) = \TV(P,Q)$. We say two distributions $P$ and $H$ are $\gamma$-close if $\TV(P,H) \leq \gamma$. We also say a distribution $P$ is $\gamma$-close to a set of distributions $\mathcal{H}$ if $\min_{H \in \mathcal{H}} \TV(P,H)\leq \gamma$. We can now formally define a distribution learner:

\begin{definition}[Realizable PAC learner]
  An algorithm is said to be a (realizable) \emph{PAC learner} for a set of distributions $\mathcal{H}$ with sample complexity $n_{\mathcal{H}}(\alpha,\beta)$ if given parameters $\alpha,\beta \in (0,1)$ and any $P\in\mathcal{H}$, the algorithm takes as input $\alpha,\beta$ and a sequence of $n_{\mathcal{H}}(\alpha,\beta)$ i.i.d.\ samples from $P$, and outputs $\widehat{H}\in\mathcal{H}$ such that $\TV(P,\widehat{H})\leq \alpha$ with probability at least $1-\beta$.\footnote{The probability is over the random samples drawn from $P^{n}$ and the randomness of the algorithm.}
\end{definition}

The following two definitions handle the case when we have model misspecification: we receive samples from a distribution $P$ which is not in the class $\mathcal{H}$.
The difference between the two is that in the robust definition (Definition~\ref{def:robust}) the algorithm is provided with an upper bound on the distance between $P$ and $\mathcal{H}$, while it is not in the agnostic setting (Definition~\ref{def:agnostic}).

\begin{definition}[$(\xi,C)$-robust PAC learner]\label{def:robust}
  An algorithm is said to be a \emph{$(\xi,C)$-robust PAC learner} for a set of distributions $\mathcal{H}$ with sample complexity $\tilde{n}_{\mathcal{H}}^{C}(\alpha,\beta)$ if given parameters $\alpha,\beta, 
  \xi \in (0,1)$ and any distribution $P$ such that $\TV(P,\mathcal{H}) \leq \xi$, the algorithm takes as input $\alpha, \beta$, $\xi$ and a sequence of $\tilde{n}_{\mathcal{H}}^{C}(\alpha,\beta)$ i.i.d.\ samples from $P$, and outputs $\widehat{H}\in\mathcal{H}$ such that $\TV(H^{*},\widehat{H})\leq C\cdot\xi+\alpha$ with probability at least $1-\beta$.
\end{definition}
\begin{definition}[$C$-agnostic PAC learner]\label{def:agnostic}
  An algorithm is said to be a \emph{$C$-agnostic PAC learner} for a set of distributions $\mathcal{H}$ with sample complexity $n_{\mathcal{H}}^{C}(\alpha,\beta)$ if for any $\alpha,\beta \in (0,1)$ and distribution $P$ such that $\TV(P,\mathcal{H})= \mathrm{OPT}$, given $\alpha$,$\beta$ and a sequence of $n_{\mathcal{H}}^{C}(\alpha,\beta)$ i.i.d.\ samples from $P$, the algorithm outputs $\widehat{H}\in\mathcal{H}$ such that $\TV(P,\widehat{H})\leq C\cdot\mathrm{OPT}+\alpha$ with probability at least $1-\beta$. If $C=1$, the algorithm is said to be agnostic.
\end{definition}
We will sometimes refer to a $C$-agnostic PAC learner as a \emph{semi-agnostic} PAC learner for $C > 1$, as is standard in learning theory. A useful object for us to define is the total variation ball.

\begin{definition}[TV ball]
The total variation ball of radius $\gamma\in[0,1]$, centered at a distribution $P$ with respect to a set of distributions $\mathcal{H}$, is the following subset of $\mathcal{H}$: $$\ball{\gamma}{P}{\mathcal{H}}=\left\{H \in \mathcal{H} : \TV(P,H) \leq \gamma \right\}.$$
\end{definition}

In this paper we consider coverings and packings of sets of distributions with respect to the total variation distance.
\begin{definition}[$\gamma$-covers and $\gamma$-packings]
  For any $\gamma\in[0,1]$ a \emph{$\gamma$-cover} of a set of distributions $\mathcal{H}$ is a set of distributions $\mathcal{C}_\gamma$, such that for every $H \in \mathcal{H}$, there exists some $P \in \mathcal{C}_\gamma$ such that $\TV(P,H) \leq \gamma$.
  
  A \emph{$\gamma$-packing} of a set of distributions $\mathcal{H}$ is a set of distributions $\mathcal{P}_\gamma \subseteq \mathcal{H}$, such that for every pair of distributions $P, Q \in \mathcal{P}_\gamma$, we have that $\TV(P,Q) \geq \gamma$.
\end{definition}

\begin{definition}[$\gamma$-covering and $\gamma$-packing number]
    For any $\gamma \in [0,1]$, the $\gamma$-\emph{covering number} of a set of distributions $\mathcal{H}$, $N(\mathcal{H},\gamma) := \min \{n\in\mathbb{N} : \exists \mathcal{C}_{\gamma} \emph{\text{ s.t. }} |\mathcal{C}_{\gamma}| = n\}$, is the size of the smallest possible $\gamma$-covering of $\mathcal{H}$. Similarly, the $\gamma$-\emph{packing number} of a set of distributions $\mathcal{H}$, $M(\mathcal{H},\gamma) := \max \{n\in\mathbb{N} : \exists \mathcal{P}_{\gamma}  \emph{\text{ s.t. }} |\mathcal{P}_{\gamma}| = n\}$, is the size of the largest subset of $\mathcal{H}$ that forms a packing for $\mathcal{H}$, 
\end{definition}
The following is a well known relation between covers and packings of a set of distribution. We defer the proof to Section~\ref{sec:proof-pack-n-cover}.
\begin{lemma}
  \label{lem:pack-n-cover}
  For a set of distributions $\mathcal{H}$ with $\gamma$-covering number $M(\mathcal{H},\gamma)$ and $\gamma$-packing number $N(\mathcal{H},\gamma)$, the following holds:
  $$M(\mathcal{H},2\gamma) \leq N(\mathcal{H},\gamma) \leq M(\mathcal{H},\gamma).$$
\end{lemma}

An important property of a set of distributions we will need to quantify in this paper is how small they are ``locally''. The following definition formalizes this:

\begin{definition}[$(k,\gamma)$-locally small]
Fix some $\gamma \in [0,1]$. We say a set of distributions $\mathcal{H}$ is $(k,\gamma)$-\emph{locally small} if $$\sup_{H'\in \mathcal{H}}|\ball{\gamma}{H'}{\mathcal{H}}| \leq k,$$ for some $k\in\mathbb{N}$. If no such $k$ exists, we say $\mathcal{H}$ is \emph{not $(k,\gamma)$-locally small}.
\end{definition}


\subsection{VC Dimension and Uniform Convergence}

An important property of a set of binary functions is its Vapnik-Chervonenkis (VC) dimension, which has the following definition:

\begin{definition}[VC dimension~\cite{VapnikC71}]\label{defn:VC}
	Let $\mathcal{F}$ be a set of binary functions $f : \mathcal{X} \to \{0,1\}$. The VC dimension of $\mathcal{F}$ is defined to be the largest $d$ such that there exist $x_1, \cdots, x_d \in \mathcal{X}$ and $f_1, \cdots, f_{2^d} \in \mathcal{F}$ such that for all $i,j \in [2^{d}]$ where $i < j$, there exists $ k \in [d]$ such that $f_i(x_k) \ne f_j(x_k)$.
\end{definition}

The most important application of the VC dimension is the following celebrated uniform convergence bound:
\begin{theorem}[Uniform Convergence \cite{talagrand1994sharper}]\label{thm:VC-unif_conv}
	Let $\mathcal{F}$ be a set of binary functions $f : \mathcal{X} \to \{0,1\}$ with VC dimension $d$. For \emph{any} distribution $P$ defined on $\mathcal{X}$, we have $$\mathbf{Pr}_{D \sim P^n} \left[{\sup_{f \in \mathcal{F}} \left|\frac{1}{n} \sum_{x \in D}f(x) - \mathbf{E}_{X\sim P}\left[f(X)\right] \right| \leq \alpha} \right] \geq 1-\beta,$$ whenever $n = O\left(\frac{d+\log(1/\beta)}{\alpha^2}\right)$.
\end{theorem}

We can define the VC dimension of a set of distributions $\mathcal{H}$ by looking at the VC dimension of a set of binary functions that is defined with respect to $\mathcal{H}$. More precisely:

\begin{definition}[VC dimension of a set of distributions]
	Let $\mathcal{H}$ be a set of probability distributions on a space $\mathcal{X}$. Define the set of binary functions $\mathcal{F}(\mathcal{H})=\{f_{H_{i},H_{j}}: H_{i},H_{j}\in\mathcal{H}\}$ where $\forall x \in \mathcal{X}$, $f_{H_{i},H_{j}}(x) = 1 \iff H_{i}(x) > H_{j}(x)$. We define the VC dimension of $\mathcal{H}$ to be the VC dimension of $\mathcal{F}(\mathcal{H})$.\footnote{To avoid measurability issues we assume the preimage of $0$ is measurable for any $f\in \mathcal{F}(\mathcal{H})$.}
\end{definition}

\begin{lemma}\label{lem:vc-lin-quad}
	The set of location Gaussians $\mathcal{G}_{d}^{L}$  has VC dimension $d+1$.  Furthermore, the set of Gaussians $\mathcal{G}_{d}$ has VC dimension $O(d^{2})$.
\end{lemma}
\begin{proof}
	For location Gaussians, $\mathcal{F}(\mathcal{G}_{d}^{L})$ corresponds to linear threshold functions (i.e., half-spaces), which have VC dimension $d+1$. Similarly $\mathcal{F}(\mathcal{G}_{d})$ corresponds to quadratic threshold functions, which have VC dimension ${d+2 \choose 2} = O(d^{2})$~\cite{Anthony95}.
\end{proof}

\subsection{Differential Privacy}
Let $X^*=\cup_{i=1}^{\infty} X^i$ be the set of possible datasets. We say that two datasets $D,D' \in X^*$ are neighbours if $D$ and $D'$ differ by at most one data point. Informally, an algorithm that receives a dataset and outputs a value in $\mathcal{R}$ is differentially private if it outputs similar values on (any) two neighboring datasets. Formally: 

\begin{definition}[\cite{DworkMNS06, DworkKMMN06}]\label{def:DP}
  A randomized algorithm $T : X^* \rightarrow \mathcal{R}$ is
  \emph{$(\eps, \delta)$-differentially private} if for all $n\geq 1$,
  for all neighbouring datasets $D,D'\in X^n$, and for all measurable
  subsets $S\subseteq \mathcal{R}$,
  $$\Pr\left[T(D)\in S\right]\leq e^\eps \Pr[T(D')\in S] + \delta\,. $$
  If $\delta = 0$, we say that $T$ is $\eps$-differentially private.
\end{definition}
We will refer to $\varepsilon$-DP as \emph{pure} DP, and $(\varepsilon, \delta)$-DP for $\delta > 0$ as \emph{approximate} DP.
A fundamental building block of differential privacy is the the exponential mechanism~\cite{McSherryT07}. It is used to privately select an approximate ``best'' candidate from a (finite) set of candidates. The quality of a candidate with respect to the dataset is measured by a score
function. Let $\mathcal{R}$ be the set of possible candidates. A score function
$S : X^* \times \mathcal{R} \rightarrow \mathbb{R}$ maps each pair consisting of a dataset
and a candidate to a real-valued score.  The \emph{exponential mechanism}
$\mathcal{M}_E$ takes as input a dataset $D$, a set of candidates $\mathcal{R}$, a score function $S$, a privacy parameter $\eps$ and outputs a candidate $r \in \mathcal{R}$ with probability proportional to
$\exp\left(\frac{\eps S(D, r)}{2\Delta(S)}\right)$, where $\Delta(S)$ is the sensitivity of the score function which is defined as
\[
  \Delta(S) = \max_{r\in \mathcal{R}, D\sim D'} \left|S(D, r) - S(D',
    r) \right|.
\]

\begin{theorem}[\cite{McSherryT07}]\label{thm:exp-mech}
  For any dataset $D$, score function $S$ and privacy parameter
  $\eps >0$, the exponential mechanism $\mathcal{M}_E(D, S, \eps)$ is an
  $\eps$-differentially private algorithm, and with probability at least
  $1 - \beta$, it selects an outcome $r\in \mathcal{R}$ such that
  \[
    S(D, r) \geq \max_{r'\in \mathcal{R}} S(D, r') - \frac{2\Delta(S)
      \log(|\mathcal{R}|/\beta)}{\eps}.
  \]

\end{theorem}

One of the most useful properties of differentially private algorithms is that they can be composed adaptively while promising a graceful degradation of privacy.

\begin{lemma}[Composition of DP]\label{lem:composition}
    If $M$ is an adaptive composition of differentially
    private algorithms $M_1,\dots,M_T$, then
    if $M_1,\dots,M_T$ are
            $(\eps_1,\delta_1),\dots,(\eps_T,\delta_T)$-differentially private
            then $M$ is $(\sum_t \eps_t , \sum_t \delta_t)$-differentially private.
\end{lemma}

Another strength of differential privacy is that it is closed under post-processing:

\begin{lemma}[Post Processing]\label{lem:post-processing}
    If $M:\mathcal{X}^n \rightarrow \mathcal{Y}$ is
    $(\eps,\delta)$-differentially private, and $P:\mathcal{Y} \rightarrow \mathcal{Z}$
    is any randomized function, then the algorithm
    $P \circ M$ is $(\eps,\delta)$-differentially private.
\end{lemma}

We now define $(\eps,\delta)$-DP learners.
\begin{definition}[$(\eps,\delta)$-DP PAC learner]

An algorithm is said to be an \emph{$(\eps,\delta)$-DP PAC learner} for a set of distributions $\mathcal{H}$ with sample complexity $n_{\mathcal{H}}(\alpha,\beta,\eps,\delta)$ if it is a PAC learner that satisfies $(\eps,\delta)$-differential privacy.
\end{definition}

\begin{definition}[$(\eps,\delta)$-DP $(\xi,C)$-robust PAC learner]
An algorithm is said to be an \emph{$(\eps,\delta)$-DP $(\xi,C)$-robust PAC learner} for a set of distributions $\mathcal{H}$ with sample complexity $\tilde{n}_{\mathcal{H}}^{C}(\alpha,\beta,\eps,\delta)$ if it is a $(\xi,C)$-robust PAC learner that satisfies $(\eps,\delta)$-differential privacy.
\end{definition}

\begin{definition}[$(\eps,\delta)$-DP $C$-agnostic PAC learner]
An algorithm is said to be an \emph{$(\eps,\delta)$-DP $C$-agnostic PAC learner} for a set of distributions $\mathcal{H}$ with sample complexity $n_{\mathcal{H}}^{C}(\alpha,\beta,\eps,\delta)$ if it is a $C$-agnostic PAC learner that satisfies $(\eps,\delta)$-differential privacy.
\end{definition}

\subsubsection{Private Hypothesis Selection and the GAP-MAX Algorithm}

The problem of \emph{hypothesis selection} (sometimes called density estimation, the Le Cam-Birg\'e method, or the Scheff\'e estimator) is a classical approach for reducing estimation problems to pairwise comparisons. 
It provides a generic approach for converting a cover for a set of probability distributions into a learning algorithm, see~\cite{DevroyeL01} for a reference.

\cite{BunKSW19} translated these powerful tools to the differentially private setting, giving an $\eps$-DP algorithm for hypothesis selection using the exponential mechanism with a carefully constructed score function. The following is a modified version where we decouple the accuracy parameter $\alpha$ from the robustness parameter $\xi$, and boost the success probability to be arbitrarily high. The proof follows immediately from the proof in~\cite{BunKSW19}. For a set of distributions $\mathcal{H}$, we will denote $H^{*}$ as the distribution in $\mathcal{H}$ that is closest to the unknown distribution $P$.

\begin{theorem}\label{theorem:PHS}
Let $\mathcal{H} = \{H_{1},\dots, H_{m}\}$ be a set of probability distributions, $\xi,\alpha, \beta,\eps,\delta \in (0,1)$ and $D \sim P^{n}$ where $P$ satisfies $\TV(P,\mathcal{H}) \leq \xi$. $\emph{\text{PHS}}(\xi,\alpha,\beta,\epsilon,\mathcal{H},D)$ is an $\eps$-DP $(\xi,3)$-robust PAC learner with sample complexity $$\tilde{n}_{\mathcal{H}}^{3}(\alpha,\beta,\eps,0) = O\left(\frac{\log(m/\beta)}{\alpha^2}+\frac{\log(m/\beta)}{\alpha\eps}\right).$$

Furthermore, when the algorithm succeeds it guarantees that $\TV(\widehat{H},H^{*}) \leq 2\xi+\alpha$.
\end{theorem}

We note the guarantee that the algorithm gives with respect to $H^{*}$ in the theorem statement for technical reasons that will become apparent in the proofs of Section~\ref{sec:learning}. Unfortunately the result above requires the number of hypotheses to be finite. Using a uniform convergence argument together with a GAP-MAX algorithm~\cite{BunDRS18}, Bun, Kamath, Steinke and Wu \cite{BunKSW19} showed that it may also be possible to get a similar guarantee when the number of hypotheses is infinite, provided that we relax the notion of privacy to approximate differential privacy. The following is an alternate version of~\cite[Theorem 4.1]{BunKSW19}. Again, in this version we decouple the accuracy parameter $\alpha$ from the robustness parameter $\xi$. The proof follows directly from the proof of Theorem 4.1 in~\cite{BunKSW19}. 

\begin{theorem}[(alternate) Theorem 4.1~\cite{BunKSW19}]\label{theorem:GAP-MAX}
Let $\mathcal{H}$ be a set of probability distributions, $\xi,\alpha,\beta,\eps,\delta \in (0,1)$ and $D \sim P^{n}$ where $P$ satisfies $\TV(P,\mathcal{H}) \leq \xi$. Furthermore, let $d$ be the VC dimension of $\mathcal{F}(\mathcal{H})$ and assume $|\ball{3\xi+\alpha}{H^{*}}{\mathcal{H}}| \leq k$. $\emph{\text{GAP-MAX}}(\alpha,\xi,\beta,\epsilon,\delta,k,\mathcal{H},D)$ is an $(\eps,\delta)$-DP $(\xi,4)$-robust PAC learner for $\mathcal{H}$ with sample complexity
$$\tilde{n}_{\mathcal{H}}^{4}(\alpha,\beta,\eps,\delta) = O\left(\frac{d+\log(1/\beta)}{\alpha^2}+\frac{\log(k/\beta)+\min\{\log(\mathcal{H}),\log(1/\delta) \}}{\alpha\eps}\right).$$

Furthermore, when the algorithm succeeds it guarantees that $\TV(\widehat{H},H^{*}) \leq 3\xi+\alpha$.
\end{theorem}

Note that Theorem~\ref{theorem:GAP-MAX} requires knowledge of $|\ball{3\xi+\alpha}{H^{*}}{\mathcal{H}}|$, which we likely do not know a priori. We can bound this by finding an upper bound on the size of the largest total variation ball centered at \emph{any} $H'\in\mathcal{H}$, i.e. $\max_{H'\in\mathcal{H}}|\ball{3\xi+\alpha}{H'}{\mathcal{H}}| \leq k'$. This directly translates to showing $\mathcal{H}$ is $(k',3\xi+\alpha)$-locally small.

This lays the foundation for the strategy used in~\cite{BunKSW19} to construct a private distribution learner for an infinite set of distributions $\mathcal{H}$: by using a $(k',6\xi+\alpha)$-locally small\footnote{Note that the guarantee we can get from any $\xi$-cover $\mathcal{C}_{\alpha}$ is $\TV(H^{*},\mathcal{C}_{\alpha}) \leq \xi \implies \TV(P,\mathcal{C}_{\alpha}) \leq 2\xi$.} $\xi$-cover for $\mathcal{H}$ as the input to the GAP-MAX algorithm, given the right amount of samples (which depends on $k'$), with high probability the algorithm outputs a distribution that is $(8\xi+\alpha$)-close to $P$.

Unfortunately, the above algorithms are not semi-agnostic and require an upper bound on $\mathrm{OPT}$ via $\xi$. As a first attempt, ~\cite{BunKSW19} give an $\eps$-DP $9$-agnostic PAC learner based on the Laplace mechanism. This algorithm -- which we will refer to as Na\"ive-PHS -- is similar to the PHS algorithm of Theorem~\ref{theorem:PHS}. Unfortunately, the sample complexity of the Na\"ive-PHS algorithm is 
$$n^{9}_{\mathcal{H}}(\alpha,\beta,\eps,0) = O\left(\frac{\log(|\mathcal{H}|/\beta)}{\alpha^{2}}+\frac{|\mathcal{H}|^{2}\log(|\mathcal{H}|/\beta)}{\alpha\eps}\right),$$which is exponentially worse than the PHS algorithm. As we will discuss shortly,~\cite{BunKSW19} also show how to use this algorithm together with the PHS algorithm to get a $\eps$-DP $18$-agnostic PAC learner, at the cost of some poly-logarithmic factors. This leads to the natural question of whether there exists an $\eps$-DP semi-agnostic learner which achieves the \emph{same} sample complexity as the PHS algorithm with a comparable agnostic constant. We answer this question in the affirmative and prove the following result:

\begin{theorem}\label{lem:private-mde}
Let $\mathcal{H} = \{H_{1},\dots, H_{m}\}$ be a set of probability distributions, $\alpha,\beta,\eps \in (0,1)$ and $D \sim P^{n}$ where $P$ satisfies $\TV(P,\mathcal{H}) = \mathrm{OPT}$. There exists an $\eps$-DP $3$-agnostic PAC learner with sample complexity $$n_{\mathcal{H}}^{3}(\alpha,\beta,\eps,0) = O\left(\frac{\log(m/\beta)}{\alpha^2}+\frac{\log(m/\beta)}{\alpha\eps}\right).$$
\end{theorem}
We defer discussing the details of this result and its proof to Section~\ref{sec:efficient-agnostic}, however we note that the above result can only handle finite sets of distributions. Recall that while the GAP-MAX algorithm can handle infinite sets of distributions, it is not a semi-agnostic learner. Thus, a natural question is whether we can learn from a set of infinite distributions using some $(\eps,\delta)$-DP semi-agnostic PAC learner.

Fortunately, \cite{BunKSW19} gave a simple procedure that takes an $(\eps,\delta)$-DP robust PAC learner and constructs an $(\eps,\delta)$-DP semi-agnostic PAC learner, at the cost of some low order poly-logarithmic factors in the sample complexity bounds, and an increase in the agnostic constant. We can thus use the GAP-MAX algorithm together with this procedure to get an $(\eps,\delta)$-DP semi-agnostic PAC leaner given an infinite set of distributions. The procedure~\cite{BunKSW19} came up with works in the following way: run the $(\eps,\delta)$-DP robust PAC learner with (a small number of) different values for $\xi$ to get a shortlist of candidates. Use the semi-agnostic Na\"ivePHS algorithm to select a good hypothesis from the short list. As we mentioned earlier, the guarantee of this approach (Theorem 3.4 in~\cite{BunKSW19}) is stated specifically in terms of converting the PHS algorithm from Theorem~\ref{theorem:PHS} into an $\eps$-DP semi-agnostic PAC learner, however it can be immediately generalized to construct $(\eps,\delta)$-DP semi-agnostic PAC learners given any $(\eps,\delta)$-DP robust PAC learner. Furthermore, we can replace the Na\"ive-PHS algorithm with the sample efficient algorithm from Theorem~\ref{lem:private-mde} to reduce the agnostic constant, and also remove some logarithmic factors in the sample complexity bound. This yields the following result:
\begin{lemma}\label{lem:semi-agnostic}
Let $\alpha,\beta,\eps,\delta \in (0,1)$ and $T=\lceil\log_{2}(1/\alpha)\rceil$. Given a set of distributions $\mathcal{H}$, an unknown distribution $P$ satisfying $\TV(P,\mathcal{H}) = \mathrm{OPT}$ and an $(\eps,\delta)$-DP $(\xi,C)$-robust PAC learner for $\mathcal{H}$ with sample complexity $\tilde{n}_{\mathcal{H}}^{C}(\alpha,\beta,\eps,\delta)$, there exists an $(\eps,\delta)$-DP $6C$-agnostic PAC learner for $\mathcal{H}$ with sample complexity
\begin{align*}
 n_{\mathcal{H}}^{6C}\left(\alpha,\beta,\eps,\delta\right) &= \tilde{n}_{\mathcal{H}}^{C}\left(\frac{\alpha}{12},\frac{\beta}{ 2(T+4)},\frac{\eps}{2(T+4)},\frac{\delta}{T+4}\right)+O\left(\frac{\log(T/\beta)}{\alpha^2}+\frac{\log(T/\beta)}{\alpha\eps}\right).
\end{align*}
\end{lemma}
We defer the proof to Section~\ref{sec:proof-lem-robust-to-agnostic}.

\section{Covering Unbounded Distributions}\label{sec:Covering}

In this section, we demonstrate a simple method to prove that a set of distributions has a locally small cover. As an application, we use this result to show that the set of \emph{unbounded} location Gaussians and scale Gaussians have locally small covers. We use these two results to give the first sample complexity result for privately learning unbounded high dimensional Gaussians in Section~\ref{sec:learning}.

\subsection{From Covering TV balls to Locally Small Covers}

The biggest roadblock to using Theorem~\ref{theorem:GAP-MAX} is demonstrating the existence of a locally small cover for the set of distributions $\mathcal{H}$. Unfortunately, explicitly constructing a \emph{global} cover (which is locally small) can be complicated, and may require cumbersome calculations even for ``simple'' distributions (see, e.g., Lemma 6.13 of~\cite{BunKSW19}). We offer a conceptually simpler alternative to prove a set of distributions $\mathcal{H}$ has a locally small cover: we demonstrate that if for every $H \in \mathcal{H}$ the total variation ball $\ball{\gamma}{H}{\mathcal{H}}$ has an $\frac{\xi}{2}$-cover of size no more than $k$, then there exists an $\xi$-cover for $\mathcal{H}$ that is $(k,\gamma)$-locally small.

\begin{lemma}\label{lem:covering-balls-implies-locally-small-cover}
Given a set of distributions $\mathcal{H}$ and $\xi \in (0,1)$, if  for \emph{every} distribution $H\in\mathcal{H}$ the total variation ball $\ball{\gamma}{H}{\mathcal{H}} \subseteq \mathcal{H}$ has an $\frac{\xi}{2}$-cover of size no more than $k$, then there exists a $(k,\gamma)$-locally small $\xi$-cover for $\mathcal{H}$.
\end{lemma}

\begin{proof}
Fix some $H\in\mathcal{H}$. By assumption, we have that the set of distributions $\ball{\gamma}{H}{\mathcal{H}}$ has an $\frac{\xi}{2}$-cover of size no more than $k$, which by definition implies that the $\frac{\xi}{2}$-covering number of $\ball{\gamma}{H}{\mathcal{H}}$ is no more than $k$. By Lemma~\ref{lem:pack-n-cover}, the $\xi$-packing number of $\ball{\gamma}{H}{\mathcal{H}}$ is also at most $k$.

Now consider an $\xi$-packing $\mathcal{P}_{\xi}$ for the set of distributions $\mathcal{H}$. We claim any such $\mathcal{P}_{\xi}$ must be $(k,\gamma)$-locally small, and we prove this by contradiction. Suppose to the contrary that there were a distribution $H'\in\mathcal{P}_{\xi}$ such that $\left|\ball{\gamma}{H'}{\mathcal{P}_{\xi}}\right| > k$. This would imply that there is an $\xi$-packing for $\ball{\gamma}{H'}{\mathcal{H}}$ with size larger than $k$, which contradicts the above observation that the packing number of \emph{any} $\ball{\gamma}{H'}{\mathcal{H}}$ is at most $k$.

A $\xi$-packing for $\mathcal{H}$ is called maximal if it is impossible to add a new element of $\mathcal{H}$ to it without violating the $\xi$-packing property. We claim that any maximal packing $\mathcal{P}'_{\xi}$ of $\mathcal{H}$ is also a $\xi$-cover of $\mathcal{H}$. We can prove this by contradiction. Suppose to the contrary that there were a distribution $P \in \mathcal{H}$ with $\TV(P, \mathcal{P}'_{\xi}) > \xi$. Then we could add $P$ to $\mathcal{P}'_{\xi}$ to produce a strictly larger packing, contradicting the maximality of $\mathcal{P}'_{\xi}$. Thus taking $\mathcal{P}_{\xi}$ to be a maximal packing gives us a $(k,\gamma)$-locally small $\xi$-cover. Therefore, it only remains to show that a maximal packing actually exists, which follows from a simple application of Zorn's Lemma.\footnote{Let $M$ be the set of all $\gamma$-packings of $\mathcal{H}$. Define a partial order on $M$ by the relation $\mathcal{P}_1 \leq \mathcal{P}_2\Longleftrightarrow \mathcal{P}_1\subseteq \mathcal{P}_2$ where $\mathcal{P}_1,\mathcal{P}_2\in M$. We claim that every chain in this partially ordered set has an upper bound in $M$; by Zorn's lemma, this would imply that $M$ has a maximal element which concludes the proof. To see why every (possibly infinite) chain $\mathcal{P}_1 \leq \mathcal{P}_2 \leq \ldots$ has an upper bound in $M$, we consider the following upper bound $U = \cup_i\mathcal{P}_i$. Note that $U \in M$ since otherwise there would be an index $i$ such that $\mathcal{P}_i\notin M$.}
\end{proof}

\subsection{Locally Small Gaussian Covers} 
We now prove that both the set of $d$-dimensional location Gaussians and scale Gaussians can be covered in a locally small fashion. Our first result shows that the set of $d$-dimensional location Gaussians $\mathcal{G}^{L}_{d}$ has a locally small cover. Our second result is proving the existence of a locally small cover for the set of $d$-dimensional scale Gaussians $\mathcal{G}^{S}_{d}$.

\subsubsection{Covering Location Gaussians}
It is not too difficult to come up with an explicit locally small cover for the set of location Gaussians without using Lemma~\ref{lem:covering-balls-implies-locally-small-cover} as is demonstrated in~\cite[Lemma 6.12]{BunKSW19}. Nonetheless we choose to do so as a warmup before attempting to solve the (harder) problem for scale Gaussians. In the case of location Gaussians, our proof is very similar to~\cite{BunKSW19}. Constructing an explicit cover is not too difficult in this case because the geometry of $\mathcal{G}_{d}^{L}$ is ``simple,'' given that the TV distance between any two distributions is determined by the $\ell_2$ distance of their means. Unfortunately the situation is not that simple in the scale Gaussian case as we will see shortly. We begin by showing that the total variation ball centered at any Gaussian $\mathcal{N}(\mu,I)$  with respect to $\mathcal{G}^{L}_{d}$ can be covered, as long as the radius is not too large.

\begin{lemma}\label{lemma-cover-TV-ball-mean}
For any $d \in \mathbb{N}$, $\mu \in \mathbb{R}^{d}$, $\xi \in (0,1)$ and $\gamma\in(\xi,c_{1})$ where $c_{1}$ is a universal constant, there exists an $\xi$-cover for the set of distributions $\ball{\gamma}{\mathcal{N}(\mu,I)}{\mathcal{G}^{L}_{d}}$ of size 
$$\left(\frac{\gamma}{\xi} \right)^{O(d)}.$$
\end{lemma}
\begin{proof}
Fix some $\mathcal{N}(\mu,I) \in \mathcal{G}^{L}_{d}$. From \cite[Theorem 1.2]{DevroyeMR18b} we have

\begin{equation}\label{eq-TV-unknown-mean}
\frac{1}{200}\min \left\{1, \|\mu_{1}-\mu_{2}\|_{2}\right\} \leq \TV\left(\mathcal{N}(\mu_{1},I),\mathcal{N}(\mu_{2},I)\right)\leq \frac{9}{2}\min \left\{1,  \|\mu_{1}-\mu_{2}\|_{2}\right\}. 
\end{equation}

For any $\gamma$ smaller than the universal constant $c_{1}$, the lower bound in~(\ref{eq-TV-unknown-mean}) implies that any $\mathcal{N}(\tilde\mu,I) \in \ball{\gamma}{\mathcal{N}(\mu,I)}{\mathcal{G}^{L}_{d}}$ must satisfy $\|\mu - \tilde\mu \|_{2} \leq 200\gamma$. We thus propose the following cover:

$$\mathcal{C}_{\xi} = \left\{\mathcal{N}(\mu + \hat{z},I) : \hat{z}\in \left(\frac{2\xi}{9\sqrt{d}}\right)\mathbb{Z}^d, \| \hat{z} \|_2 \leq 200\gamma\right \}. $$


We now prove that $\mathcal{C}_{\xi}$ is a valid $\xi$-cover. Fix some $\mathcal{N}(\tilde\mu,I) \in \ball{\gamma}{\mathcal{N}(\mu,I)}{\mathcal{G}^{L}_{d}}$ and define $z =  \tilde\mu - \mu$. We know $\|z \|_2\leq 200\gamma$.
Let $\hat{z} = \left(\frac{2\xi}{9\sqrt{d}}\right)\lfloor\left(\frac{9\sqrt{d}}{2\xi}\right)z\rfloor$ and
$\hat{\mu} = \mu + \hat{z}$.
Note that we have $\mathcal{N}(\hat\mu,I)\in\mathcal{C}_{\xi}$. Furthermore, $z$ and $\hat{z}$ are element-wise close ($\|z-\hat{z}\|_\infty\leq 2\xi/9\sqrt{d}$) therefore we have
\begin{align*}
\TV\left(\mathcal{N}(\tilde\mu,I),\mathcal{N}(\hat\mu,I)\right) &\leq \frac{9}{2} \| \tilde\mu - \hat\mu \|_{2} \\
&= \frac{9}{2} \|z - \hat{z}\|_{2}\\ 
&\leq \frac{9\sqrt{d}}{2} \| z - \hat{z}\|_{\infty}\\
&\leq \frac{9\sqrt{d}}{2}\cdot\frac{2\xi}{9\sqrt{d}} \\
&= \xi,
\end{align*}


where the first inequality follows from~(\ref{eq-TV-unknown-mean}). 
We now bound the size of this cover.

\begin{align*}
\left|\mathcal{C}_{\xi}\right| &= \left| \left\{\mathcal{N}(\mu + \hat{z},I) : \hat{z}\in \left(\frac{2\xi}{9\sqrt{d}}\right)\mathbb{Z}^d, \| \hat{z} \|_2 \leq 200\gamma\right \}\right| \\  
&\leq \left| \left\{\hat{z}: \hat{z}\in \mathbb{Z}^d, \| \hat{z} \|_2 \leq \frac{900\sqrt{d}\gamma}{\xi}\right \}\right|
\leq \left| \left\{\hat{z}: \hat{z}\in \mathbb{Z}^d, \| \hat{z} \|_1 \leq \frac{900d\gamma}{\xi}\right \}\right|\\
&\leq \left| \left\{z_1 - z_2: z_1,z_2\in \mathbb{Z}_{+}^d, \| z_1 \|_1 \leq \left\lceil\frac{900d\gamma}{\xi}\right\rceil, \| z_2 \|_1 \leq \left\lceil\frac{900d\gamma}{\xi}\right\rceil \right \}\right|\\
&\leq \left| \left\{z: z\in \mathbb{Z}_{+}^d, \| z \|_1 \leq \left\lceil\frac{900d\gamma}{\xi}\right\rceil\right \}\right|^2 \\
&\leq \left(\sum_{i=1}^{\left\lceil900d\gamma/\xi\right\rceil} { i + d - 1\choose {d - 1}}
\right)^2\\
 &\leq\left(\left\lceil\frac{900d\gamma}{\xi}\right\rceil {\left\lceil900d\gamma/\xi\right\rceil + d - 1\choose {d - 1}}
\right)^2 \leq \left(\frac{\gamma}{\xi} \right)^{O(d)},
\end{align*}
where the third last inequality follows from the standard solution to the stars and bars problem.
\end{proof}

Combining Lemma~\ref{lem:covering-balls-implies-locally-small-cover} with Lemma~\ref{lemma-cover-TV-ball-mean} immediately gives us the following corollary:

\begin{corollary}\label{corollary:covering-location-Gaussians}
For any $d \in \mathbb{N}$, $\xi \in (0,1)$ and $\gamma \in (\xi,c_{1})$ where $c_{1}$ is a universal constant, there exists an $\xi$-cover $\mathcal{C}_{\xi}$ for the set of $d$-dimensional location Gaussians $\mathcal{G}^{L}_{d}$ that is $\left((2\gamma/\xi)^{O(d)},\gamma\right)$-locally small.
\end{corollary}

\subsubsection{Covering Scale Gaussians}
It is not a trivial exercise to come up with an explicit cover for the class of scale Gaussians due to the complicated nature of the geometry of $\mathcal{G}_{d}^{S}$. Fortunately for us, Lemma~\ref{lem:covering-balls-implies-locally-small-cover} simplifies things significantly. It turns out that if we want to cover the TV ball centered at any $\mathcal{N}(0,\Sigma)$, we can use a cover for $\mathcal{N}(0,I)$ and ``stretch'' the covariance matrices of every distribution in the cover (using $\Sigma$) so that the modified cover becomes a valid cover for the TV ball centered at $\mathcal{N}(0,\Sigma)$.  The following lemma tells us that we can cover the total variation ball centered at $\mathcal{N}(0,I)$ with respect to $\mathcal{G}^{S}_{d}$ as long as the radius is not too large.
\begin{lemma}\label{Lemma-Covering-TV-Ball-Std-Normal}
For any $d \in \mathbb{N}$, $\xi \in (0,1)$ and $\gamma \in (\xi,c_{2})$ where $c_{2}$ is a universal constant, there exists an $\xi$-cover for the set of distributions $\ball{\gamma}{\mathcal{N}(0,I)}{\mathcal{G}^{S}_{d}}$ of size 
$$\left(\frac{\gamma}{\xi}\right)^{O(d^2)}.$$
\end{lemma}
\begin{proof}

From \cite[Theorem 1.1]{DevroyeMR18b} we have,
\begin{equation}\label{TV-unknown-covaraince-eig}
 \TV(\mathcal{N}(0,I),\mathcal{N}(0,\Sigma)) \geq \frac{1}{100}\min\left\{1,\sqrt{\sum_{i=1}^{d}\lambda_{i}^2} \right\},
\end{equation}
where $\lambda_{1}\dots\lambda_{d}$ are the eigenvalues of $\Sigma-I$, and it holds that $\sqrt{\sum_{i=1}^{d}\lambda_{i}^2} = \|\Sigma - I\|_{F}$. 

For any $\gamma$ smaller than the universal constant $c_{2}'$
, the lower bound in (\ref{TV-unknown-covaraince-eig}) implies two things: 1) for any $\mathcal{N}(0,\Sigma) \in \ball{\gamma}{\mathcal{N}(0,I)}{\mathcal{G}^{S}}$, $\|\Sigma - I \|_{F} \leq 100\gamma$ and 2) the minimum eigenvalue of $\Sigma$, $\lambda_{\text{min}}$, satisfies $\lambda_{\text{min}}\geq 1-100\gamma$.
We thus propose the following cover: 
$$\mathcal{C}_{\xi} = \left\{\mathcal{N}(0, I + \hat{\Delta}) : \hat{\Delta}\in \rho\mathbb{Z}^{d\times d}\cap \PD , \|\hat{\Delta}\|_F\leq 100\gamma\right\}, $$
where $\rho=\frac{\xi\sqrt{2\pi e}(1-100\gamma)}{d+\xi\sqrt{2\pi e}}$. First we will show that this is a valid cover. Consider an arbitrary $\mathcal{N}(0,\Sigma) \in \ball{\gamma}{\mathcal{N}(0,I)}{\mathcal{G}^{S}_{d}}$. We want to show that there is a distribution $\mathcal{N}(0,\hat{\Sigma}) \in \mathcal{C}_{\xi}$ that is $\xi$-close to $\mathcal{N}(0,\Sigma)$. Let $\Delta = \Sigma - I$, let $\hat{\Delta}=\rho\lfloor\Delta/\rho\rfloor$ and let $\hat{\Sigma} = I + \hat{\Delta}$. Since $\|\Delta\|_{F} = \|\Sigma-I\|_{F} \leq 100\gamma$, $\mathcal{N}(0,\hat{\Sigma})$ is indeed in the cover.

 
Next we show that $\TV(\mathcal{N}(0,\Sigma),\mathcal{N}(0,\hat{\Sigma}))\leq \xi$. We use Proposition~32 in~\cite{ValiantV10a}, which states for any two positive definite matrices $\Sigma$ and $\hat{\Sigma}$, if $\|\Sigma-\hat{\Sigma}\|_{\infty,\infty}\leq \rho'$ and the smallest eigenvalue of $\Sigma$ satisfies $\lambda_\text{min} >\eta$, then we have

\begin{equation}\label{Valiant-TV-Upperbound}
\TV(\mathcal{N}(0,\Sigma),\mathcal{N}(0,\hat{\Sigma}))\leq \frac{d\rho'}{\sqrt{2\pi e}(\eta-\rho')}.
\end{equation}

By the definition of $\mathcal{C}_{\xi}$, $\|\hat{\Delta}-\Delta\|_{\infty,\infty} \leq \rho$. Since any valid $\Sigma$ must satisfy $\lambda_\text{min}\geq 1-100\gamma$, our choice of setting $\rho=\frac{\xi\sqrt{2\pi e}(1-100\gamma)}{d+\xi\sqrt{2\pi e}}$ implies that \begin{align*}
\TV(\mathcal{N}(0,\Sigma),\mathcal{N}(0,\hat{\Sigma}))\leq \xi,
\end{align*} for any $\gamma$ smaller than the universal constant 
$c_{2}'$. We now bound the size of the cover in a similar manner to the case of location Gaussians.

\begin{align*}
\left|\mathcal{C}_{\xi}\right| &= \left|\left\{\hat{\Delta}\in \rho\mathbb{Z}^{d\times d}\cap \PD : \|\hat{\Delta}\|_{F}\leq 100\gamma\right\}\right|\\
 &\leq \left|\left\{\hat{\Delta}\in \mathbb{Z}^{d\times d}: \|\hat{\Delta}\|_{F}\leq 100\gamma/\rho\right\}\right|\\
 &\leq \left|\left\{\hat{\Delta}\in \mathbb{Z}^{d\times d}: \|\hat{\Delta}\|_{1,1}\leq 100\gamma d/\rho\right\}\right|\\
 &\leq \left|\left\{\hat{\Delta}\in \mathbb{Z}_{+}^{d\times d}: \|\hat{\Delta}\|_{1,1}\leq \left\lceil 100\gamma d/\rho \right\rceil \right\}\right|^2\\
 &\leq \left(\left\lceil\frac{100\gamma d(d+\xi\sqrt{2\pi e})}{\xi\sqrt{2\pi e}(1-100\gamma)}\right\rceil\cdot{\left\lceil\frac{100\gamma d(d+\xi\sqrt{2\pi e})}{\xi\sqrt{2\pi e}(1-100\gamma)}\right\rceil + d^2 - 1 \choose {d^2 - 1}}
\right)^2,
 \end{align*} 
for any $\gamma$ smaller than the universal constant $c_{2}''$ we have,
\begin{align*}
\left|\mathcal{C}_{\xi}\right| \leq \left(\frac{\gamma}{\xi}\right)^{O(d^2)}.
 \end{align*} 
Setting $c_{2} = \max\{c_{2}',c_{2}''\}$ completes the proof.
\end{proof}

The following corollary is a direct consequence of Lemma~\ref{Lemma-Covering-TV-Ball-Std-Normal} and Proposition~\ref{prop:mappingTV}.

\begin{corollary}\label{corollary:Covering-scale-Gaussian-ball}
For any $d \in \mathbb{N}$, $\xi \in (0,1)$, $\gamma \in (0,c_{2})$ where $c_{2}$ is a universal constant and $\Sigma \in \PD$, there exists an $\xi$-cover for the set of distributions $\ball{\gamma}{\mathcal{N}(0,\Sigma}{\mathcal{G}^{S}_{d})}$ of size $$\left(\frac{\gamma}{\xi}\right)^{O(d^2)}.$$
\end{corollary}

\begin{proof}
Fix some $\Sigma \in \PD$ and define $\Sigma^{1/2}$ as one of its matrix square-roots. By Proposition~\ref{prop:mappingTV} and Equation~(\ref{eq:Gaussiantransform}) we have:

\begin{equation}\label{eq:TV_scale_gaussian}
\TV\left(\mathcal{N}\left(0,\Sigma^{1/2}\Sigma_{1}\Sigma^{1/2}\right),\mathcal{N}\left(0,\Sigma^{1/2}\Sigma_{2}\Sigma^{1/2}\right)\right) \leq \TV\left(\mathcal{N}\left(0,\Sigma_{1}\right),\mathcal{N}\left(0,\Sigma_{2}\right)\right).
\end{equation}

We can thus take the cover $\mathcal{C}_{\xi}$ in Lemma~\ref{Lemma-Covering-TV-Ball-Std-Normal}, and replace every distribution $\mathcal{N}(0,\Sigma_{1}) \in \mathcal{C}_{\xi}$ with $\mathcal{N}(0,\Sigma^{1/2}\Sigma_{1}\Sigma^{1/2})$. Note that our modified cover will have the same size. From~(\ref{eq:TV_scale_gaussian}), our new cover will be a valid $\xi$-cover for $\ball{\gamma}{\mathcal{N}(0,\Sigma)}{\mathcal{G}^{S}_{d}}$ since the TV distance can not increase between any two distributions $\mathcal{N}(0,\Sigma_{1}), \mathcal{N}(0,\Sigma_{2}) \in \ball{\gamma}{\mathcal{N}(0,I)}{\mathcal{G}^{S}_{d}}$ after applying the transformation above. 
\end{proof}

We can now combine Lemma~\ref{lem:covering-balls-implies-locally-small-cover} with Corollary~\ref{corollary:Covering-scale-Gaussian-ball} to get the following:

\begin{corollary}\label{corollary:covering-scale-Gaussians}
For any $d \in \mathbb{N}$, $\xi \in (0,1)$, and $\gamma \in (\xi,c_{2})$ where $c_{2}$ is a universal constant, there exists an $\xi$-cover $\mathcal{C}_{\xi}$ for the set of scale Gaussians $\mathcal{G}^{S}_{d}$ that is $\left((2\gamma/\xi)^{O(d^2)},\gamma\right)$-locally small.
\end{corollary}

\section{Beyond GAP-MAX: Boosting Weak Hypotheses}\label{sec:learning}

As we mentioned before, by using a $(k,6\xi+\alpha)$-locally small $\xi$-cover for an infinite set of distributions $\mathcal{H}$, one can utilize Theorem~\ref{theorem:GAP-MAX} to privately learn a distribution to low error. Unfortunately, this approach will yield a sample complexity bound that has a term of order $O(\log(1/\delta)/\alpha\eps)$. In the case of learning an unbounded univariate Gaussian in the realizable setting, it is known that the sample complexity is $O(1/\alpha^{2} +\log(1/\delta)/\eps)$~\cite{KarwaV18}, however the upper bound on the sample complexity achieved by Theorem~\ref{theorem:GAP-MAX} (together with an appropriate locally smaller cover) is $O(1/\alpha^{2}+\log(1/\delta)/\alpha\eps)$~\cite[Corollary 6.15]{BunKSW19}. In order to overcome the poor dependence on $\log(1/\delta)$, we can instead aim for a two step approach:
\begin{enumerate}
    \item Use the GAP-MAX algorithm in Theorem~\ref{theorem:GAP-MAX} but with constant accuracy $C$ to learn a distribution $H'$ that is roughly $C$-close to the true Gaussian for some appropriately selected constant $C<1$.
    
    \item Build a \emph{finite} cover for $\ball{C}{H'}{\mathcal{H}}$ and use the private hypothesis selection algorithm (Theorem~\ref{theorem:PHS}) to learn a distribution $\widehat{H}$ that is $\alpha$-close to the true Gaussian.
\end{enumerate}

Running the GAP-MAX algorithm with constant accuracy $C$ thus removes the dependence on $\alpha$ in the $O(\log(1/\delta)/\eps)$ term. Intuitively, this approach learns a ``rough'' estimate of the right distribution using approximate differential privacy. Since we know that we are roughly $C$-close to the true Gaussian, we can cover $\ball{C}{H'}{\mathcal{H}}$ with a finite cover, and use the $\eps$-differentially private hypothesis selection algorithm. This two step approach which we dub \emph{boosting} gets us a much better dependence on the privacy parameter $\delta$ in our sample complexity bounds, and as we will see it holds more generally in the robust learning setting.

\begin{remark}
\label{rem:coarse-cover}
We note that the first step in the above approach may only need to produce an exceptionally coarse estimate to the true distribution -- one to which it bears very little resemblance at all! 
We illustrate this with the simple problem of privately estimating a univariate Gaussian $\mathcal{N}(\mu, 1)$ (in the realizable case).

We work backwards: our overall target is an algorithm with sample complexity $\tilde O(1/\alpha^2 + 1/\alpha\varepsilon + \log(1/\delta)/\varepsilon)$.
Since using the pure DP hypothesis selection algorithm of Theorem~\ref{theorem:PHS} takes  $O(\log |\mathcal{C}_\alpha| (1/\alpha^2 + 1/\alpha\varepsilon))$ samples, we require only that $|\mathcal{C}_\alpha|$ is less than some quasi-polynomial in $1/\alpha$.
For the sake of exposition, suppose we restrict further and require $|\mathcal{C}_\alpha| \leq 1/\alpha^{101}$.
This can be achieved by starting at any point which is at most $1/\alpha^{100}$ from the true mean $\mu$ and taking an $\alpha$-additive grid over this space.
But if we only require a starting point $\hat \mu$ which is $1/\alpha^{100}$-close to the true mean $\mu$, this corresponds (by Gaussian tail bounds) to a distribution whose total variation distance is roughly $1 - \exp(-1/\alpha^{200})$ with respect to the true distribution.

We can see that the first step in the procedure truly requires an exceptionally coarse estimate of the distribution.
The estimate of the mean described is significantly further from the true mean than any individual point will be.
Interestingly, note that if one requires a more accurate final distribution, the distribution output in the first step is allowed to be less accurate.
\end{remark}

\subsection{Warmup: Learning Location Gaussians}\label{sec:learn-location-gaussians}

As a first step, we can show that Algorithm~\ref{alg:boosting-Location} can achieve a slightly more general guarantee than a robust PAC learning sample complexity bound. We make Algorithm~\ref{alg:boosting-Location} more general than it needs to be to give a robust learning guarantee for $\mathcal{G}^{L}_{d}$ in order to make use of it as a subroutine in Algorithm~\ref{alg:boosting-General} which robustly learns $\mathcal{G}_{d}$.

\begin{lemma}\label{lem:Location-Gaussian-sample complexity-generalized-robust}
Let $b$ be a positive constant. For any $\beta,\eps,\delta \in (0,1)$, $\alpha \in (0,c_{3})$ and $\xi\in(0,c_{4})$ where $c_{3}$ and $c_{4}$ are constants that depend only on $b$, given a dataset $D\sim P^{n}$ where $P$ satisfies $\TV(P,\mathcal{G}^{L}_{d})\leq b\xi + \alpha/4$, $\emph{\text{BOOST}}_{1}\left(b,\xi,\alpha,\beta,\eps,\delta,\mathcal{C}_{\alpha},D\right)$ is an $(\eps,\delta)$-differentially private algorithm which outputs some $\widehat{H}\in\mathcal{G}^{L}_{d}$ such that $\TV(\widehat{H},P)\leq 3(b+1)\xi+\alpha$ with probability no less than $1-\beta$, so long as $$n=O\left(\frac{d+\log(1/\beta)}{\alpha^2} + \frac{d+\log(1/\beta)}{\alpha\eps} + \frac{\log(1/\beta\delta)}{\eps} \right).$$
\end{lemma}

\begin{algorithm}[H]
\caption{Boosting for learning $\mathcal{G}^{L}_{d}$: $\text{BOOST}_{1}(b,\xi,\alpha,\beta,\eps,\delta,\mathcal{C}_{\alpha},D)$.}
\label{alg:boosting-Location}
\DontPrintSemicolon
\SetKwInOut{Input}{Input}\SetKwInOut{Output}{Output}
\setstretch{1.35}
\Input{Positive constant $b$, robustness parameter $\xi \in (0,1)$, accuracy parameters $\alpha,\beta \in (0, 1)$, privacy parameters $\eps,\delta \in (0,1)$, locally small $\xi$-cover $\mathcal{C}_{\xi}$ for $\mathcal{G}^{L}_{d}$ and dataset $D$ of size $n$.}
\Output{Distribution $\widehat{H} \in \mathcal{G}^{L}_{d}.$}

$H' = \text{GAP-MAX}\left(b\xi+\frac{\alpha}{4},\left(\frac{1}{200(600b+1)}-\frac{3\alpha}{4}\right),\beta/2,\eps/2,\delta,k,\mathcal{C}_{\alpha},D\right)$ \tcp*{$H'=\mathcal{N}(\mu',I)$}\label{alg:boosting-location-line1}
  
Build an $\alpha$-cover $\tilde{\mathcal{C}}_{\xi}$ for $\ball{\left(3b\xi+\frac{1}{200(600b+1)}\right)}{H'}{\mathcal{G}^{L}_{d}}$ \label{alg:boosting-location-line2}
  
\textbf{Return} $\widehat{H} = \text{PHS}((b+1)\xi+\alpha,\alpha,\beta/2,\eps/2,\tilde{\mathcal{C}}_{\alpha},D)$ \tcp*{$\widehat{H}=\mathcal{N}(\hat{\mu},I) $} \label{alg:boosting-location-line3}

\end{algorithm}

\begin{proof}[Proof of Lemma~\ref{lem:Location-Gaussian-sample complexity-generalized-robust}]
We first show Algorithm~\ref{alg:boosting-Location} satisfies $(\eps,\delta)$-differential privacy. Line~\ref{alg:boosting-location-line1} of the algorithm is $(\eps/2,\delta)$-differentially private by the guarantee of Theorem~\ref{theorem:GAP-MAX}. Line~\ref{alg:boosting-location-line2} maintains $(\eps/2,\delta)$-privacy by post-processing (Lemma~\ref{lem:post-processing}). Finally, line~\ref{alg:boosting-location-line3} is $(\eps/2,0)$-differentially private by Theorem~\ref{theorem:PHS}. By composition (Lemma~\ref{lem:composition}), the entire algorithm is $(\eps,\delta)$-differentially private.

We now argue about the accuracy of the algorithm. By Lemma~\ref{lem:vc-lin-quad}, Corollary~\ref{corollary:covering-location-Gaussians} and Theorem~\ref{theorem:GAP-MAX}, as long as $n=O\left(\frac{d+\log(1/\beta\delta)}{\eps} \right)$, with probability no less than $1-\beta/2$ the GAP-MAX algorithm in line~\ref{alg:boosting-location-line1} outputs a distribution $H'$ that is $\left(3b\xi+\frac{1}{200(600b+1)}\right)$-close to $H^{*}$, for any $\xi$ and $\alpha$ smaller than constants $c_{3}'$ and $c_{4}$, respectively, that depend only on $b$. For the remainder of the proof we condition on this event.

In line~\ref{alg:boosting-location-line2}, we build an $\xi$-cover $\tilde{\mathcal{C}}_{\xi}$ for $\ball{\left(3b\xi+\frac{1}{200(600b+1)}\right)}{H'}{\mathcal{G}^{L}_{d}}$ that satisfies $\TV(H^{*},\tilde{\mathcal{C}}_{\xi})\leq\xi$. By the triangle inequality $\TV(P,\tilde{\mathcal{C}}_{\xi})\leq (b+1)\xi+\alpha/4$. By Lemma~\ref{lemma-cover-TV-ball-mean}, we can indeed construct $\tilde{\mathcal{C}}_{\xi}$ such that $|\tilde{\mathcal{C}}_{\xi}| \leq (C)^{O(d)}$ (for some constant $C$) as long as $\xi$ is smaller than a constant $c_{3}''$ that depends only on $b$. By the stated accuracy and size of $\tilde{\mathcal{C}}_{\xi}$, with probability no less than $1-\beta/2$, Theorem~\ref{theorem:PHS} guarantees that line~\ref{alg:boosting-location-line3} outputs $\widehat{H}$ satisfying $$\TV(P,\widehat{H}) \leq 3((b+1)\xi+\alpha/4)+\alpha/4 = 3(b+1)\xi+\alpha,$$ as long as $n=O\left(\frac{d+\log(1/\beta)}{\alpha^2} + \frac{d+\log(1/\beta)}{\alpha\eps}\right)$. Setting $n=O\left(\frac{d+\log(1/\beta)}{\alpha^2} + \frac{d)+\log(1/\beta)}{\alpha\eps} + \frac{\log(1/\beta\delta)}{\eps} \right)$ together with a union bound completes the proof.
\end{proof}

The following result can be derived from the $\text{BOOST}_{1}$ algorithm by taking the standard assumption in the robust setting of $\TV(P,\mathcal{H}) \leq \xi$. The proof is nearly identical to the proof of Lemma~\ref{lem:Location-Gaussian-sample complexity-generalized-robust}.
\begin{lemma}\label{lem:Location-gaussian-sample-complexity-robust}
Let $\alpha,\beta,\eps,\delta \in(0,1)$ and $\xi\in(0,c_{5})$ where $c_5$ is a universal constant, and let $D\sim P^{n}$ where $P$ satisfies $\TV(P,\mathcal{G}^{L}_{d})\leq \xi$. Furthermore, let $\mathcal{C}_{\xi}$ be an appropriately selected locally small $\xi$-cover for $\mathcal{G}_{d}^{L}$. $\emph{\text{BOOST}}_{1}\left(1,\xi,\alpha,\beta,\eps,\delta,\mathcal{C}_{\alpha},D\right)$ is an $(\eps,\delta)$-DP $(\xi,6)$-robust PAC learner for $\mathcal{G}_{d}^{L}$ with sample complexity $$\tilde{n}_{\mathcal{G}_{d}^{L}}^{6}(\alpha,\beta,\eps,\delta)=O\left(\frac{d+\log(1/\beta)}{\alpha^2} + \frac{d+\log(1/\beta)}{\alpha\eps} + \frac{\log(1/\beta\delta)}{\eps} \right).$$
\end{lemma}

We can now use Lemma~\ref{lem:semi-agnostic} together with Lemma~\ref{lem:Location-gaussian-sample-complexity-robust} to get a semi-agnostic algorithm.

\begin{theorem}\label{thm:Location-gaussian-sample-complexity-agnostic}
Let $\alpha,\beta,\eps,\delta \in(0,1)$ and $D\sim P^{n}$ where $P$ satisfies $\TV(P,\mathcal{G}^{L}_{d}) = \mathrm{OPT}$. For any $\mathrm{OPT}$ smaller than a universal constant $c_{5}$, there exists an $(\eps,\delta)$-DP $36$-agnostic PAC learner for $\mathcal{G}_{d}^{L}$ with sample complexity $$n_{\mathcal{G}_{d}^{L}}^{36}(\alpha,\beta,\eps,\delta)=\widetilde{O}\left(\frac{d+\log(1/\beta)}{\alpha^2} + \frac{d+\log(1/\beta)}{\alpha\eps} + \frac{\log(1/\beta\delta)}{\eps} \right).$$
\end{theorem}

\subsection{Learning Gaussians}\label{sec:learn-gaussians}

We can now show that Algorithm~\ref{alg:boosting-General} achieves the following sample complexity bound for robust learning.
\begin{lemma}\label{lemma:Gaussian-sample complexity-robust}
Let $\beta,\eps,\delta \in (0,1)$, $\alpha\in(0,c_{6})$ and $\xi\in(0,c_{7})$ where $c_{6}$ and $c_{7}$ are universal constants, and let $D\sim P^{n}$ where $P$ satisfies $\TV(P,\mathcal{G}^{L}_{d})\leq \xi$. Furthermore, let $\mathcal{C}_{\alpha}^{1}$ and $\mathcal{C}_{\alpha}^{2}$ be appropriately selected locally small covers for $\mathcal{G}_{d}$ and $\mathcal{G}_{d}^{L}$ respectively. $\emph{\text{BOOST}}_{2}\left(\xi,\alpha,\beta,\eps,\delta,\mathcal{C}_{\alpha}^{1},\mathcal{C}_{\alpha}^{2},D\right)$ is an $(\eps,\delta)$-DP $(\xi,33)$-robust PAC learner for $\mathcal{G}_{d}$ with sample complexity $$\tilde{n}_{\mathcal{G}_{d}}^{33}(\alpha,\beta,\eps,\delta)=O\left(\frac{d^{2}+\log(1/\beta)}{\alpha^2} + \frac{d^{2}+\log(1/\beta)}{\alpha\eps} + \frac{\log(1/\beta\delta)}{\eps} \right).$$
\end{lemma}

\begin{algorithm}[H]
    \caption{Boosting for learning $\mathcal{G}_{d}$: $\text{BOOST}_{2}(\xi,\alpha,\beta,\eps,\delta,\mathcal{C}^{1}_{\alpha},\mathcal{C}^{2}_{\alpha},D)$.}
    \label{alg:boosting-General}
    
    \DontPrintSemicolon
    \SetKwInOut{Input}{Input}\SetKwInOut{Output}{Output}
    
    \Input{Robustness parameter $\xi \in (0,1)$, accuracy parameters $\alpha,\beta \in (0, 1)$, privacy parameters $\eps,\delta \in (0,1)$, locally small $\xi$-cover $\mathcal{C}^{1}_{\xi}$ for $\mathcal{G}^{S}_{d}$, locally small $\xi$-cover $\mathcal{C}^{2}_{\xi}$ for $\mathcal{G}^{L}_{d}$ and dataset $D$ of size $2n$.}
    \setstretch{1.35}
    \Output{Distribution $\widehat{H}\in\mathcal{G}_{d}.$}
    
     Set $D' = \{Y_{1},\dots,Y_{n}\}$ where $Y_{i} = (X_{2i}-X_{2i-1})/\sqrt{2}$\label{alg:boosting-general-line1}
    
    $H_{1}' = \text{GAP-MAX}\left(4\xi,\frac{1}{100(1201)},\beta/4,\eps/4,\delta/2,k_{1},\mathcal{C}^{1}_{\alpha},D'\right)$ \tcp*{$H_{1}' = \mathcal{N}(0,\Sigma')$}\label{alg:boosting-general-line2}
    
    Build an $\xi$-cover $\tilde{\mathcal{C}}^{1}_{\xi}$ for $\ball{\left(12\xi+\frac{1}{100(1201)}\right)}{H_{1}'}{\mathcal{G}^{S}_{d}}$\label{alg:boosting-general-line3}
  
    $\widehat{H}_{1} = \text{PHS}(4\xi,\alpha,\beta/4,\eps/4,\tilde{\mathcal{C}}^{1}_{\alpha},D')$ \tcp*{ $\widehat{H}_{1} = \mathcal{N}(0,\widehat{\Sigma})$}\label{alg:boosting-general-line4}
    
    Set $D'' = \{W_{1},\dots,W_{2n}\}$ where $W_{i} = \widehat{\Sigma}^{-1/2}X_{i}$\label{alg:boosting-general-line5}
    
    $\widehat{H}_{2} = \text{BOOST}_{1}\left(10,\xi,\alpha,\beta/2,\eps/2,\delta/2,\mathcal{C}^{2}_{\alpha},D''\right)$ \tcp*{$\widehat{H}_{2} = \mathcal{N}(\hat{\mu},I)$}\label{alg:boosting-general-line6}
    
    \textbf{Return:} $\widehat{H} = \mathcal{N}\left(\widehat{\Sigma}^{1/2}\widehat{\mu},\widehat{\Sigma}\right)$
\end{algorithm}

\begin{proof}[Proof of Lemma~\ref{lemma:Gaussian-sample complexity-robust}]
We first show Algorithm~\ref{alg:boosting-General} satisfies $(\eps,\delta)$-differential privacy. Line~\ref{alg:boosting-general-line2} of the algorithm is $(\eps/4,\delta/2)$-differentially private by the guarantee of Theorem~\ref{theorem:GAP-MAX}. Line~\ref{alg:boosting-general-line3} maintains $(\eps/4,\delta/2)$-privacy by post-processing(Lemma~\ref{lem:post-processing}). Line~\ref{alg:boosting-general-line4} is $(\eps/4,0)$-differentially private by Theorem~\ref{theorem:PHS}. Line~\ref{alg:boosting-general-line5} maintains privacy by post processing (Lemma~\ref{lem:post-processing}). Finally, line~\ref{alg:boosting-general-line6} is $(\eps/2,\delta/2)$-differentially private by the privacy of Algorithm~\ref{alg:boosting-Location} proved in Lemma~\ref{lem:Location-Gaussian-sample complexity-generalized-robust}. By composition (Lemma~\ref{lem:composition}) the entire algorithm is $(\eps,\delta)$-differentially private.

We now argue about the accuracy of the algorithm. Let $H^{*}=\mathcal{N}(\mu^{*},\Sigma^{*})$ be the hypothesis in $\mathcal{G}_{d}$ that satisfies $\TV(H^{*},P)\leq\xi$. By Lemma~\ref{lem:centering-approx-gaussian} $Y_{i}\sim Q$ where $\TV(Q,\mathcal{N}(0,\Sigma^{*}))\leq 3\xi$, which implies that $\TV(\mathcal{N}(0,\Sigma^{*}),\mathcal{C}_{\xi}^{1})\leq 4\xi$. From Lemma~\ref{lem:vc-lin-quad}, Corollary~\ref{corollary:covering-scale-Gaussians} and Theorem~\ref{theorem:GAP-MAX}, as long as $n=O\left(\frac{d^{2}+\log(1/\beta\delta)}{\eps} \right)$, with probability no less than $1-\beta/4$ the GAP-MAX algorithm in line~\ref{alg:boosting-general-line2} outputs a distribution $H_{1}'=\mathcal{N}(0,\Sigma')$ that is $\left(12\xi+\frac{1}{100(1201)}\right)$-close to $\mathcal{N}(0,\Sigma^{*})$, for any $\xi$ smaller than a universal constant $c_{6}'$. We condition on this event.

In line~\ref{alg:boosting-general-line3}, we build a $\xi$-cover $\tilde{\mathcal{C}}_{\xi}^{1}$ for $\ball{12\xi+\frac{1}{100(1201)}}{H_{1}'}{\mathcal{G}^{S}_{d}}$ that satisfies $\TV(\mathcal{N}(0,\Sigma^{*}),\tilde{\mathcal{C}_{\xi}^{1}}) \leq \xi$. It follows immediately that $\TV(Q,\tilde{\mathcal{C}_{\xi}^{1}}) \leq 4\xi$ from the triangle inequality. By Corollary~\ref{corollary:Covering-scale-Gaussian-ball},  we can indeed construct $\tilde{\mathcal{C}}_{\xi}^{1}$ such that $|\tilde{\mathcal{C}}_{\xi}^{1}| \leq (C)^{O(d^{2})}$ for any $\xi$ smaller than a universal constant $c_{6}''$. By the stated accuracy and size of $\tilde{\mathcal{C}}_{\xi}^{1}$ Theorem~\ref{theorem:PHS} guarantees, with probability at least $1-\beta/4$, that line~\ref{alg:boosting-general-line4} outputs $\widehat{H}_{1} =\mathcal{N}(0,\widehat{\Sigma}) $ such that $\TV(\widehat{H}_{1},\mathcal{N}(0,\Sigma^{*})) \leq 8\xi+\alpha/4$ as long as $n=O\left(\frac{d^{2}+\log(1/\beta)}{\alpha^2} + \frac{d^{2}+\log(1/\beta)}{\alpha\eps}\right)$. We further condition on this event. 

Let $R$ be the distribution satisfying $W_{i} \sim R$.\footnote{If $\widehat{\Sigma}$ is not invertible, the range of $\widehat{\Sigma}$ is an $r$-dimensional linear subspace of $\mathbb{R}^{d}$, for some $r<d$. Let $\Pi$ be a $d\times r$ matrix whose columns form a basis for the range of $\widehat{\Sigma}$. It follows that $\widetilde{\Sigma} = \Pi^{T}\widehat{\Sigma}\Pi$ is a valid $r$-dimensional full-rank covariance matrix. Moreover, by construction, $\widetilde{\Sigma}$ is identical to $\widehat{\Sigma}$ after projection on to the subspace defined by the range of $\widehat{\Sigma}$. We can thus project our data onto the range of $\widehat{\Sigma}$ and continue the algorithm using $\widetilde{\Sigma}$.} By Corollary~\ref{corollary:invertmappingTV}, $\TV(R,\mathcal{N}(\mu^{*},\widehat{\Sigma}^{-1/2}\Sigma^{*}\widehat{\Sigma}^{-1/2})) \leq \xi$. By application of the triangle inequality, Corollary~\ref{corollary:invertmappingTV} and Equation~(\ref{eq:Gaussiantransform}),
\begin{align*}
    TV(R,\mathcal{N}(\mu^{*},I)) &\leq \TV(R,\mathcal{N}(\mu^{*},\widehat\Sigma^{-1/2}\Sigma^{*}\widehat\Sigma^{-1/2})) + \TV(\mathcal{N}(\mu^{*},\widehat\Sigma^{-1/2}\Sigma^{*}\widehat\Sigma^{-1/2}),\mathcal{N}(\mu^{*},I)) \nonumber \\
    &\leq\xi + \TV(\mathcal{N}(\mu^{*},\Sigma^{*}),\mathcal{N}(\mu^{*},\widehat{\Sigma}))\nonumber \\
    &= \xi + \TV(\mathcal{N}(0,\Sigma^{*}),\widehat{H}_{1}) \nonumber \\
    &\leq 9\xi+\alpha/4.
\end{align*}

It follows from the triangle inequality that $\TV(R,\mathcal{C}_{\xi}^{2}) \leq 10\xi+\alpha/4$. This together with Lemma~\ref{lem:Location-Gaussian-sample complexity-generalized-robust} implies that, with probability greater than $1-\beta/2$, $\text{BOOST}_{1}(10,\xi,\alpha,\beta/2,\eps/2,\delta/2,\mathcal{C}_{\alpha}^{2},D'')$ outputs $\widehat{H}_{2} = \mathcal{N}(\hat\mu,I)$ satisfying $\TV(\widehat{H}_{2},\mathcal{N}(\mu^{*},I)) \leq 33\xi+\alpha$ for any $\xi$ and $\alpha$ smaller than universal constants $c_{6}'''$ and $c_{7}$ respectively, as long as $n=O\left(\frac{d+\log(1/\beta)}{\alpha^2} + \frac{d+\log(1/\beta)}{\alpha\eps} + \frac{\log(1/\beta\delta)}{\eps} \right)$. Using the triangle inequality and Corollary~\ref{corollary:invertmappingTV} it follows that $\TV(\mathcal{N}(\widehat{\Sigma}^{1/2}\hat{\mu},\widehat{\Sigma}),P) \leq 33\xi+\alpha$. Setting $n=O\left(\frac{d^{2}+\log(1/\beta)}{\alpha^2} + \frac{d^{2}+\log(1/\beta)}{\alpha\eps} + \frac{\log(1/\beta\delta)}{\eps} \right)$ and $c_{6} = \max\left\{c_{6}',c_{6}'',c_{6}'''\right\}$ together with a union bound completes the proof.
\end{proof}

Finally, we can combine the above result with Lemma~\ref{lem:semi-agnostic} to get a semi-agnostic sample complexity bound for modest levels of model misspecification. 

\begin{theorem}\label{thm:Gaussian-sample complexity-agnostic}
Let $\beta, \eps,\delta \in (0,1)$, $\alpha \in(0,c_{6})$ for some universal constant $c_{6}$, and let $D\sim P^{n}$ where $P$ satisfies $\TV(P,\mathcal{G}^{L}_{d}) = \mathrm{OPT}$. For any $\mathrm{OPT}$ smaller than a universal constant $c_{7}$, there exists an $(\eps,\delta)$-DP $198$-agnostic PAC learner for $\mathcal{G}_{d}$ with sample complexity $$n_{\mathcal{G}_{d}}^{198}(\alpha,\beta,\eps,\delta)=\widetilde{O}\left(\frac{d^{2}+\log(1/\beta)}{\alpha^2} + \frac{d^{2}+\log(1/\beta)}{\alpha\eps} + \frac{\log(1/\beta\delta)}{\eps} \right).$$

\end{theorem}

\subsection{Bounds for the Realizable Setting}

The following sample complexity bounds hold for $(\eps,\delta)$-DP (realizable) PAC learning. The proofs are very similar to the proofs in Section~\ref{sec:learn-location-gaussians} and~\ref{sec:learn-gaussians} for $(\eps,\delta)$-DP robust PAC learning, where the slight difference is that we can build the covers directly with accuracy $\alpha$ (instead of $\xi$) since we assume realizability. The first bound is tight and the second one is conjectured to be tight.
\begin{lemma}
For any $\beta, \eps,\delta \in (0,1)$ and $\alpha \in (0,c_{8})$ where $c_{8}$ is a universal constant, there exists an $(\eps,\delta)$-DP PAC learner for $\mathcal{G}_{d}^{L}$ with sample complexity
$$n_{\mathcal{G}_{d}^{L}}(\alpha,\beta,\eps,\delta)=O\left(\frac{d+\log(1/\beta)}{\alpha^2} + \frac{d+\log(1/\beta)}{\alpha\eps} + \frac{\log(1/\beta\delta)}{\eps} \right).$$
\end{lemma}

\begin{lemma}
For any $\beta, \eps,\delta \in (0,1)$ and $\alpha \in (0,c_{9})$ where $c_{9}$ is a universal constant, there exists an $(\eps,\delta)$-DP PAC learner for $\mathcal{G}_{d}$ with sample complexity
$$n_{\mathcal{G}_{d}}(\alpha,\beta,\eps,\delta)=O\left(\frac{d^{2}+\log(1/\beta)}{\alpha^2} + \frac{d^{2}+\log(1/\beta)}{\alpha\eps} + \frac{\log(1/\beta\delta)}{\eps} \right).$$
\end{lemma}

\section{Agnostic Private Hypothesis Selection}\label{sec:efficient-agnostic}
In this section we present an $\eps$-DP semi-agnostic PAC learner that achieves the same sample complexity as the PHS algorithm of Theorem~\ref{theorem:PHS}. The PHS algorithm is based on the celebrated \emph{Scheff\'e tournament} (see, e.g., Chapter 6 of~\cite{DevroyeL01}), where the distributions in $\mathcal{H}$ play a round robin tournament against one another. The winner of this tournament is then chosen as the output. One of the technical difficulties in constructing privatized versions of the Scheff\'e tournament via the exponential mechanism is that a single sample can quite drastically change the outcome of the tournament, which makes choosing score functions based on tournaments challenging. We sidestep this issue completely by considering another approach to hypothesis selection called the \emph{minimum distance estimate} (MDE). The MDE approach is based on \emph{maximizing} a particular function of the data and $\mathcal{H}$ as we will see shortly. Fortunately, this estimator is already in the form of a maximization problem and the function we aim to maximize has low sensitivity. Thus, using the exponential mechanism together with the MDE is a very natural way to privatize semi-agnostic hypothesis selection. The MDE requires $O(m^{3})$ computations, where $m$ is the number of hypotheses in $\mathcal{H}$.~\cite{MahalanabisS08} presented a modified MDE that is very similar to the original MDE, but only requires $O(m^{2})$ computations. Fortunately, this modified algorithm maintains the guarantee of the original algorithm, so we will privatize the modified MDE instead of the original MDE. We formally state our result below.

\begin{reptheorem}{lem:private-mde}
Let $\mathcal{H} = \{H_{1},\dots, H_{m}\}$ be a set of probability distributions, $\alpha,\beta,\eps \in (0,1)$ and $D \sim P^{n}$ where $P$ satisfies $\TV(P,\mathcal{H}) = \mathrm{OPT}$. There exists an $\eps$-DP $3$-agnostic PAC learner with sample complexity $$n_{\mathcal{H}}^{3}(\alpha,\beta,\eps,0) = O\left(\frac{\log(m/\beta)}{\alpha^2}+\frac{\log(m/\beta)}{\alpha\eps}\right).$$
\end{reptheorem}

Before we prove the result, we define a few things. For an ordered pair of distributions $(H_{i},H_{j})$ over a common domain $\mathcal{X}$, we define their \emph{Scheff\'e set} as $A_{ij} = \{ x \in \mathcal{X} : H_{i}(x) > H_{j}(x)\}$. A useful version of the TV distance between two distributions we will make use of is $$2\TV(H_{i},H_{j}) = \big(H_{i}(A_{ij}) - H_{j}(A_{ij})\big) + \big(H_{j}(A_{ji}) - H_{i}(A_{ji})\big).$$ We note that most of the analysis below is standard in proving the correctness of the MDE (e.g., see the proof of Theorem 6.3 in~\cite{DevroyeL01}), and is slightly adapted using the analysis of the modified MDE algorithm in Theorem 4 of~\cite{MahalanabisS08}. The only difference here is the use of the exponential mechanism.

\begin{proof}[Proof of Theorem~\ref{lem:private-mde}]
Let $\mathcal{X}$ be the common domain of the distributions in $\mathcal{H}$. For a dataset $D$ and set $A \subseteq \mathcal{X}$, we define $\widehat{P}(A,D) = \frac{1}{n}\cdot|\{x\in D: x \in A\}|$. For a distribution $H$ and a set $A\subseteq \mathcal{X}$, let $R(H,A) = H(A)-\widehat{P}(A)$. For any $H_i \in \mathcal{H}$, we define the score function 
\begin{align*}
S(D,H_{i}) &= -\sup_{j\in[m]\backslash\{i\}} \left|\big(H_{i}(A_{ij}) - \widehat{P}(A_{ij},D)\big) - \big( H_{i}(A_{ji}) - \widehat{P}(A_{ji},D)\big)\right|\\
&= -\sup_{j\in[m]\backslash\{i\}} \left|R(H_{i},A_{ij}) - R(H_{i},A_{ji})\right|.
\end{align*}
With this in place, the algorithm is simple: run the exponential mechanism~\cite{McSherryT07} with this score function, on the set of candidates $\mathcal{H}$, with dataset $D$, and return whichever distribution it outputs.

It is not hard to see that the score function has sensitivity $2/n$.
Let $H_k\in\mathcal{H}$ be any distribution that maximizes the score function. From Theorem~\ref{thm:exp-mech}, it follows that running the exponential mechanism with our dataset $D$, the set of distributions $\mathcal{H}$, privacy parameter $\eps$ and the the score function above outputs a distribution $H_{k'}\in\mathcal{H}$ that guarantees, with probability no less than $1-\beta/2$,

\begin{align*}
S(D,H_{k'}) &\geq S(D,H_{k}) - \frac{4\log(2m/\beta)}{n\eps}\\
&\geq S(D,H_{k})-\alpha/2,
\end{align*}
where the last line holds so long as $n = O\left(\frac{\log(m/\beta)}{\alpha\eps}\right)$. We condition on this event, which can equivalently be stated as 
\begin{align}\label{eq:mde-score}
\sup_{j\in[m]\backslash\{k'\}}\left|R(H_{k'},A_{k'j})-R(H_{k'},A_{jk'})\right| \leq \sup_{j\in[m]\backslash\{k\}}\left| R(H_{k},A_{kj})-R(H_{k},A_{jk})\right|+\alpha/2.
\end{align}

We can now bound the total variation distance between the unknown distribution $P$ and the output $H_{k'}$. Let $H_{l}$ be any distribution in $\mathcal{H}$ that satisfies $\TV(H_{l},P) =\mathrm{OPT}$.\footnote{Note that this implies that $\|H_{l} - P \|_{1} = 2\mathrm{OPT}$.} Using the triangle inequality we have,

\begin{align}
2\TV(H_{k'},P)&= \|H_{k'} - P\|_{1} \nonumber\\
&\leq  \|H_{l} - P\|_{1}+ \|H_{k'} - H_{l}\|_{1}.\label{eq:triangle-opt-estimate}
\end{align}

We now look at the right most term in~(\ref{eq:triangle-opt-estimate}). By the definition of the total variation distance and an application of the triangle inequality we have
\begin{align*}
\|H_{k'} - H_{l}\|_{1} &=   \big(H_{k'}(A_{k'l}) - H_{l}(A_{k'l})\big) + \big(H_{l}(A_{lk'}) - H_{k'}(A_{lk'})\big)\\
&= \left|\big(H_{k'}(A_{k'l}) - H_{k'}(A_{lk'})\big) + \big(H_{l}(A_{lk'}) - H_{l}(A_{k'l})\big) \right| \\
&\leq \left|R(H_{k'},A_{k'l})-R(H_{k'},A_{lk'})\right|+\left|R(H_{l},A_{lk'})-R(H_{l},A_{k'l})\right|\\
&\leq \sup_{j \in [m]\backslash\{k'\}} \left|R(H_{k'},A_{k'j})-R(H_{k'},A_{jk'})\right| +  \sup_{j \in [m]\backslash\{l\}} \left|R(H_{l},A_{lj})-R(H_{l},A_{jl})\right|.
\end{align*}
Using~(\ref{eq:mde-score}), the fact that $H_{k}$ maximizes the score function, and the triangle inequality all together yields,
\begin{align*}
\|H_{k'} - H_{l}\|_{1} &\leq \sup_{j \in [m]\backslash\{k\}} \left|R(H_{k},A_{kj})-R(H_{k},A_{jk})\right| +  \sup_{j \in [m]\backslash\{l\}} \left|R(H_{l},A_{lj})-R(H_{l},A_{jl})\right| + \alpha\\
&\leq 2\sup_{j \in [m]\backslash\{l\}} \left|R(H_{l},A_{lj})-R(H_{l},A_{jl})\right| + \alpha\\
&\leq 2\sup_{j \in [m]\backslash\{l\}} \left| \big(H_{l}(A_{lj}) - P(A_{lj})\big) + \big(P(A_{jl}) - H_{l}(A_{jl})\big)\right| \\
&\hspace*{0.03cm} + 2\sup_{j \in [m]\backslash\{l\}} \left|\big(P(A_{lj}) - \widehat{P}(A_{lj},D)\big) + \big(\widehat{P}(A_{jl},D) - P(A_{jl})\big)\right| + \alpha.
\end{align*}
Notice that the first term on the right hand side of the final inequality is at most twice the $\ell_{1}$ distance between $H_{l}$ and $P$. Let $\Delta(P) = 2\sup_{j \in [m]\backslash\{l\}} \left|\big(P(A_{lj}) - \widehat{P}(A_{lj},D)\big) + \big(\widehat{P}(A_{jl},D) - P(A_{jl})\big)\right|$. This gives us
\begin{align*}
\|H_{k'} - H_{l}\|_{1} &\leq 2\|H_{l} - P\|_{1} + \Delta(P) + \alpha.
\end{align*}

Furthermore, notice that the term $\Delta(P)$ is small when the difference between the empirical and the true probability measures assigned by $P$ to the Scheff\`e sets is small. We can thus upper bound this term by $\alpha$ by using $2{{m}\choose{2}}$ standard Chernoff bounds together with a union bound to get,
\begin{align}\label{eq:final-ineq}
\|H_{k'} - H_{l}\|_{1} &\leq 2\|H_{l} - P\|_{1} + 2\alpha,
\end{align}
with probability no less than $1-\beta/2$ so long as $n=O\left(\frac{\log(m/\beta)}{\alpha^{2}}\right)$.
Putting~(\ref{eq:triangle-opt-estimate}) and~(\ref{eq:final-ineq}) together gives us,
\begin{align*}
\TV(H_{k'},P) &\leq \frac{3}{2}\|H_{l} - P \|_{2}  + \alpha\\
&=3\mathrm{OPT} + \alpha.
\end{align*}
A union bound together with setting $n=O\left(\frac{\log(m/\beta)}{\alpha^{2}}+\frac{\log(m/\beta)}{\alpha\eps}\right)$ completes the proof.
\end{proof}

\section{Conclusion}
We provide the first finite sample complexity bounds for privately learning Gaussians with unbounded parameters. 
We do this via a method for converting small local covers, to global covers which are locally small.
In this paper, we only prove sample complexity upper bounds, and our methods are not computational in nature.
One natural direction is to design polynomial time algorithms for learning unbounded Gaussians.
Another direction is to explore applications of our method to other classes of distributions.
The most immediate class that comes to mind is Gaussians mixture models (GMMs) -- given upper and lower bounds on the total variation distance between GMMs based on their parameter distance, it should not be difficult to derive corresponding sample complexity bounds.

\section*{Acknowledgments}
GK would like to thank Mark Bun, Adam Smith, Thomas Steinke, and Zhiwei Steven Wu for helpful conversations and suggestions which led to the results in Section~\ref{sec:efficient-agnostic}.

\bibliography{main, biblio}

\newcommand{\etalchar}[1]{$^{#1}$}
\begin{thebibliography}{ABDH{\etalchar{+}}18}

\bibitem[AAA20]{aden2020sample}
Ishaq Aden-Ali and Hassan Ashtiani.
\newblock On the sample complexity of learning sum-product networks.
\newblock In {\em International Conference on Artificial Intelligence and
  Statistics}, pages 4508--4518. PMLR, 2020.

\bibitem[ABDH{\etalchar{+}}18]{AshtianiBHLMP18}
Hassan Ashtiani, Shai Ben-David, Nicholas Harvey, Christopher Liaw, Abbas
  Mehrabian, and Yaniv Plan.
\newblock Nearly tight sample complexity bounds for learning mixtures of
  {G}aussians via sample compression schemes.
\newblock In {\em Advances in Neural Information Processing Systems 31},
  NeurIPS '18, pages 3412--3421. Curran Associates, Inc., 2018.

\bibitem[ABDM18]{ashtiani2017sample}
Hassan Ashtiani, Shai Ben-David, and Abbas Mehrabian.
\newblock Sample-efficient learning of mixtures.
\newblock In {\em Proceedings of the Thirty-Second AAAI Conference on
  Artificial Intelligence}, AAAI'18, pages 2679--2686. AAAI Publications, 2018.

\bibitem[AFJ{\etalchar{+}}18]{AcharyaFJOS18}
Jayadev Acharya, Moein Falahatgar, Ashkan Jafarpour, Alon Orlitsky, and
  Ananda~Theertha Suresh.
\newblock Maximum selection and sorting with adversarial comparators.
\newblock {\em Journal of Machine Learning Research}, 19(1):2427--2457, 2018.

\bibitem[AJOS14]{AcharyaJOS14b}
Jayadev Acharya, Ashkan Jafarpour, Alon Orlitsky, and Ananda~Theertha Suresh.
\newblock Sorting with adversarial comparators and application to density
  estimation.
\newblock In {\em Proceedings of the 2014 IEEE International Symposium on
  Information Theory}, ISIT '14, pages 1682--1686, Washington, DC, USA, 2014.
  IEEE Computer Society.

\bibitem[Ant95]{Anthony95}
Martin Anthony.
\newblock Classification by polynomial surfaces.
\newblock {\em Discrete Applied Mathematics}, 61(2):91--103, 1995.

\bibitem[ASZ20]{AcharyaSZ20}
Jayadev Acharya, Ziteng Sun, and Huanyu Zhang.
\newblock Differentially private assouad, fano, and le cam.
\newblock {\em arXiv preprint arXiv:2004.06830}, 2020.

\bibitem[BBKN14]{BeimelBKN14}
Amos Beimel, Hai Brenner, Shiva~Prasad Kasiviswanathan, and Kobbi Nissim.
\newblock Bounds on the sample complexity for private learning and private data
  release.
\newblock {\em Machine Learning}, 94(3):401--437, 2014.

\bibitem[BDKU20]{BiswasDKU20}
Sourav Biswas, Yihe Dong, Gautam Kamath, and Jonathan Ullman.
\newblock Coinpress: Practical private mean and covariance estimation.
\newblock {\em arXiv preprint arXiv:2006.06618}, 2020.

\bibitem[BDRS18]{BunDRS18}
Mark Bun, Cynthia Dwork, Guy~N. Rothblum, and Thomas Steinke.
\newblock Composable and versatile privacy via truncated cdp.
\newblock In {\em Proceedings of the 50th Annual ACM Symposium on the Theory of
  Computing}, STOC '18, pages 74--86, New York, NY, USA, 2018. ACM.

\bibitem[BEM{\etalchar{+}}17]{BittauEMMRLRKTS17}
Andrea Bittau, {\'U}lfar Erlingsson, Petros Maniatis, Ilya Mironov, Ananth
  Raghunathan, David Lie, Mitch Rudominer, Ushasree Kode, Julien Tinnes, and
  Bernhard Seefeld.
\newblock Prochlo: Strong privacy for analytics in the crowd.
\newblock In {\em Proceedings of the 26th ACM Symposium on Operating Systems
  Principles}, SOSP '17, pages 441--459, New York, NY, USA, 2017. ACM.

\bibitem[BKM19]{BousquetKM19}
Olivier Bousquet, Daniel~M. Kane, and Shay Moran.
\newblock The optimal approximation factor in density estimation.
\newblock In {\em Proceedings of the 32nd Annual Conference on Learning
  Theory}, COLT '19, pages 318--341, 2019.

\bibitem[BKSW19]{BunKSW19}
Mark Bun, Gautam Kamath, Thomas Steinke, and Zhiwei~Steven Wu.
\newblock Private hypothesis selection.
\newblock In {\em Advances in Neural Information Processing Systems 32},
  NeurIPS '19, pages 156--167. Curran Associates, Inc., 2019.

\bibitem[BNS16]{BunNS16}
Mark Bun, Kobbi Nissim, and Uri Stemmer.
\newblock Simultaneous private learning of multiple concepts.
\newblock In {\em Proceedings of the 7th Conference on Innovations in
  Theoretical Computer Science}, ITCS '16, pages 369--380, New York, NY, USA,
  2016. ACM.

\bibitem[BS16]{BunS16}
Mark Bun and Thomas Steinke.
\newblock Concentrated differential privacy: Simplifications, extensions, and
  lower bounds.
\newblock In {\em Proceedings of the 14th Conference on Theory of
  Cryptography}, TCC '16-B, pages 635--658, Berlin, Heidelberg, 2016. Springer.

\bibitem[BS19]{BunS19}
Mark Bun and Thomas Steinke.
\newblock Average-case averages: Private algorithms for smooth sensitivity and
  mean estimation.
\newblock In {\em Advances in Neural Information Processing Systems 32},
  NeurIPS '19, pages 181--191. Curran Associates, Inc., 2019.

\bibitem[BSU17]{BunSU17}
Mark Bun, Thomas Steinke, and Jonathan Ullman.
\newblock Make up your mind: The price of online queries in differential
  privacy.
\newblock In {\em Proceedings of the 28th Annual ACM-SIAM Symposium on Discrete
  Algorithms}, SODA '17, pages 1306--1325, Philadelphia, PA, USA, 2017. SIAM.

\bibitem[BUV14]{BunUV14}
Mark Bun, Jonathan Ullman, and Salil Vadhan.
\newblock Fingerprinting codes and the price of approximate differential
  privacy.
\newblock In {\em Proceedings of the 46th Annual ACM Symposium on the Theory of
  Computing}, STOC '14, pages 1--10, New York, NY, USA, 2014. ACM.

\bibitem[CWZ19]{CaiWZ19}
T.~Tony Cai, Yichen Wang, and Linjun Zhang.
\newblock The cost of privacy: Optimal rates of convergence for parameter
  estimation with differential privacy.
\newblock {\em arXiv preprint arXiv:1902.04495}, 2019.

\bibitem[DDS12]{DaskalakisDS12b}
Constantinos Daskalakis, Ilias Diakonikolas, and Rocco~A. Servedio.
\newblock Learning {P}oisson binomial distributions.
\newblock In {\em Proceedings of the 44th Annual ACM Symposium on the Theory of
  Computing}, STOC '12, pages 709--728, New York, NY, USA, 2012. ACM.

\bibitem[DFM{\etalchar{+}}20]{DuFMBG20}
Wenxin Du, Canyon Foot, Monica Moniot, Andrew Bray, and Adam Groce.
\newblock Differentially private confidence intervals.
\newblock {\em arXiv preprint arXiv:2001.02285}, 2020.

\bibitem[DHS15]{DiakonikolasHS15}
Ilias Diakonikolas, Moritz Hardt, and Ludwig Schmidt.
\newblock Differentially private learning of structured discrete distributions.
\newblock In {\em Advances in Neural Information Processing Systems 28}, NIPS
  '15, pages 2566--2574. Curran Associates, Inc., 2015.

\bibitem[{Dif}17]{AppleDP17}
{Differential Privacy Team, Apple}.
\newblock Learning with privacy at scale.
\newblock
  \url{https://machinelearning.apple.com/docs/learning-with-privacy-at-scale/appledifferentialprivacysystem.pdf},
  December 2017.

\bibitem[DK14]{DaskalakisK14}
Constantinos Daskalakis and Gautam Kamath.
\newblock Faster and sample near-optimal algorithms for proper learning
  mixtures of {G}aussians.
\newblock In {\em Proceedings of the 27th Annual Conference on Learning
  Theory}, COLT '14, pages 1183--1213, 2014.

\bibitem[DKK{\etalchar{+}}16]{DiakonikolasKKLMS16}
Ilias Diakonikolas, Gautam Kamath, Daniel~M. Kane, Jerry Li, Ankur Moitra, and
  Alistair Stewart.
\newblock Robust estimators in high dimensions without the computational
  intractability.
\newblock In {\em Proceedings of the 57th Annual IEEE Symposium on Foundations
  of Computer Science}, FOCS '16, pages 655--664, Washington, DC, USA, 2016.
  IEEE Computer Society.

\bibitem[DKM{\etalchar{+}}06]{DworkKMMN06}
Cynthia Dwork, Krishnaram Kenthapadi, Frank McSherry, Ilya Mironov, and Moni
  Naor.
\newblock Our data, ourselves: Privacy via distributed noise generation.
\newblock In {\em Proceedings of the 24th Annual International Conference on
  the Theory and Applications of Cryptographic Techniques}, EUROCRYPT '06,
  pages 486--503, Berlin, Heidelberg, 2006. Springer.

\bibitem[DKY17]{DingKY17}
Bolin Ding, Janardhan Kulkarni, and Sergey Yekhanin.
\newblock Collecting telemetry data privately.
\newblock In {\em Advances in Neural Information Processing Systems 30}, NIPS
  '17, pages 3571--3580. Curran Associates, Inc., 2017.

\bibitem[DL96]{DevroyeL96}
Luc Devroye and G\'abor Lugosi.
\newblock A universally acceptable smoothing factor for kernel density
  estimation.
\newblock {\em The Annals of Statistics}, 24(6):2499--2512, 1996.

\bibitem[DL97]{DevroyeL97}
Luc Devroye and G\'abor Lugosi.
\newblock Nonasymptotic universal smoothing factors, kernel complexity and
  {Y}atracos classes.
\newblock {\em The Annals of Statistics}, 25(6):2626--2637, 1997.

\bibitem[DL01]{DevroyeL01}
Luc Devroye and G\'abor Lugosi.
\newblock {\em Combinatorial methods in density estimation}.
\newblock Springer, 2001.

\bibitem[DL09]{DworkL09}
Cynthia Dwork and Jing Lei.
\newblock Differential privacy and robust statistics.
\newblock In {\em Proceedings of the 41st Annual ACM Symposium on the Theory of
  Computing}, STOC '09, pages 371--380, New York, NY, USA, 2009. ACM.

\bibitem[DLS{\etalchar{+}}17]{DajaniLSKRMGDGKKLSSVA17}
Aref~N. Dajani, Amy~D. Lauger, Phyllis~E. Singer, Daniel Kifer, Jerome~P.
  Reiter, Ashwin Machanavajjhala, Simson~L. Garfinkel, Scot~A. Dahl, Matthew
  Graham, Vishesh Karwa, Hang Kim, Philip Lelerc, Ian~M. Schmutte, William~N.
  Sexton, Lars Vilhuber, and John~M. Abowd.
\newblock The modernization of statistical disclosure limitation at the {U.S.}
  census bureau, 2017.
\newblock Presented at the September 2017 meeting of the Census Scientific
  Advisory Committee.

\bibitem[DMNS06]{DworkMNS06}
Cynthia Dwork, Frank McSherry, Kobbi Nissim, and Adam Smith.
\newblock Calibrating noise to sensitivity in private data analysis.
\newblock In {\em Proceedings of the 3rd Conference on Theory of Cryptography},
  TCC '06, pages 265--284, Berlin, Heidelberg, 2006. Springer.

\bibitem[DMR18]{DevroyeMR18b}
Luc Devroye, Abbas Mehrabian, and Tommy Reddad.
\newblock The total variation distance between high-dimensional {G}aussians.
\newblock {\em arXiv preprint arXiv:1810.08693}, 2018.

\bibitem[DN03]{DinurN03}
Irit Dinur and Kobbi Nissim.
\newblock Revealing information while preserving privacy.
\newblock In {\em Proceedings of the 22nd ACM SIGMOD-SIGACT-SIGART Symposium on
  Principles of Database Systems}, PODS '03, pages 202--210, New York, NY, USA,
  2003. ACM.

\bibitem[DR16]{DworkR16}
Cynthia Dwork and Guy~N. Rothblum.
\newblock Concentrated differential privacy.
\newblock {\em arXiv preprint arXiv:1603.01887}, 2016.

\bibitem[DSS{\etalchar{+}}15]{DworkSSUV15}
Cynthia Dwork, Adam Smith, Thomas Steinke, Jonathan Ullman, and Salil Vadhan.
\newblock Robust traceability from trace amounts.
\newblock In {\em Proceedings of the 56th Annual IEEE Symposium on Foundations
  of Computer Science}, FOCS '15, pages 650--669, Washington, DC, USA, 2015.
  IEEE Computer Society.

\bibitem[DSSU17]{DworkSSU17}
Cynthia Dwork, Adam Smith, Thomas Steinke, and Jonathan Ullman.
\newblock Exposed! a survey of attacks on private data.
\newblock {\em Annual Review of Statistics and Its Application}, 4(1):61--84,
  2017.

\bibitem[EPK14]{ErlingssonPK14}
{\'U}lfar Erlingsson, Vasyl Pihur, and Aleksandra Korolova.
\newblock {RAPPOR}: Randomized aggregatable privacy-preserving ordinal
  response.
\newblock In {\em Proceedings of the 2014 ACM Conference on Computer and
  Communications Security}, CCS '14, pages 1054--1067, New York, NY, USA, 2014.
  ACM.

\bibitem[GKK{\etalchar{+}}20]{GopiKKNWZ20}
Sivakanth Gopi, Gautam Kamath, Janardhan Kulkarni, Aleksandar Nikolov,
  Zhiwei~Steven Wu, and Huanyu Zhang.
\newblock Locally private hypothesis selection.
\newblock In {\em Proceedings of the 33rd Annual Conference on Learning
  Theory}, COLT '20, 2020.

\bibitem[HSR{\etalchar{+}}08]{HomerSRDTMPSNC08}
Nils Homer, Szabolcs Szelinger, Margot Redman, David Duggan, Waibhav Tembe,
  Jill Muehling, John~V. Pearson, Dietrich~A. Stephan, Stanley~F. Nelson, and
  David~W. Craig.
\newblock Resolving individuals contributing trace amounts of {DNA} to highly
  complex mixtures using high-density {SNP} genotyping microarrays.
\newblock {\em PLoS Genetics}, 4(8):1--9, 2008.

\bibitem[HT10]{HardtT10}
Moritz Hardt and Kunal Talwar.
\newblock On the geometry of differential privacy.
\newblock In {\em Proceedings of the 42nd Annual ACM Symposium on the Theory of
  Computing}, STOC '10, pages 705--714, New York, NY, USA, 2010. ACM.

\bibitem[KKMN09]{KorolovaKMN09}
Aleksandra Korolova, Krishnaram Kenthapadi, Nina Mishra, and Alexandros
  Ntoulas.
\newblock Releasing search queries and clicks privately.
\newblock In {\em Proceedings of the 18th International World Wide Web
  Conference}, WWW '09, pages 171--180, New York, NY, USA, 2009. ACM.

\bibitem[KLSU19]{KamathLSU19}
Gautam Kamath, Jerry Li, Vikrant Singhal, and Jonathan Ullman.
\newblock Privately learning high-dimensional distributions.
\newblock In {\em Proceedings of the 32nd Annual Conference on Learning
  Theory}, COLT '19, pages 1853--1902, 2019.

\bibitem[KSSU19]{KamathSSU19}
Gautam Kamath, Or~Sheffet, Vikrant Singhal, and Jonathan Ullman.
\newblock Differentially private algorithms for learning mixtures of separated
  {G}aussians.
\newblock In {\em Advances in Neural Information Processing Systems 32},
  NeurIPS '19, pages 168--180. Curran Associates, Inc., 2019.

\bibitem[KSU20]{KamathSU20}
Gautam Kamath, Vikrant Singhal, and Jonathan Ullman.
\newblock Private mean estimation of heavy-tailed distributions.
\newblock In {\em Proceedings of the 33rd Annual Conference on Learning
  Theory}, COLT '20, 2020.

\bibitem[KU20]{KamathU20}
Gautam Kamath and Jonathan Ullman.
\newblock A primer on private statistics.
\newblock {\em arXiv preprint arXiv:2005.00010}, 2020.

\bibitem[KV18]{KarwaV18}
Vishesh Karwa and Salil Vadhan.
\newblock Finite sample differentially private confidence intervals.
\newblock In {\em Proceedings of the 9th Conference on Innovations in
  Theoretical Computer Science}, ITCS '18, pages 44:1--44:9, Dagstuhl, Germany,
  2018. Schloss Dagstuhl--Leibniz-Zentrum fuer Informatik.

\bibitem[LSY{\etalchar{+}}20]{LiuSYKR20}
Yuhan Liu, Ananda~Theertha Suresh, Felix Yu, Sanjiv Kumar, and Michael Riley.
\newblock Learning discrete distributions: User vs item-level privacy.
\newblock {\em arXiv preprint arXiv:2007.13660}, 2020.

\bibitem[MS08]{MahalanabisS08}
Satyaki Mahalanabis and Daniel Stefankovic.
\newblock Density estimation in linear time.
\newblock In {\em Proceedings of the 21st Annual Conference on Learning
  Theory}, COLT '08, pages 503--512, 2008.

\bibitem[MT07]{McSherryT07}
Frank McSherry and Kunal Talwar.
\newblock Mechanism design via differential privacy.
\newblock In {\em Proceedings of the 48th Annual IEEE Symposium on Foundations
  of Computer Science}, FOCS '07, pages 94--103, Washington, DC, USA, 2007.
  IEEE Computer Society.

\bibitem[NRS07]{NissimRS07}
Kobbi Nissim, Sofya Raskhodnikova, and Adam Smith.
\newblock Smooth sensitivity and sampling in private data analysis.
\newblock In {\em Proceedings of the 39th Annual ACM Symposium on the Theory of
  Computing}, STOC '07, pages 75--84, New York, NY, USA, 2007. ACM.

\bibitem[NS18]{NissimS18}
Kobbi Nissim and Uri Stemmer.
\newblock Clustering algorithms for the centralized and local models.
\newblock In {\em Algorithmic Learning Theory}, ALT '18, pages 619--653. JMLR,
  Inc., 2018.

\bibitem[NSV16]{NissimSV16}
Kobbi Nissim, Uri Stemmer, and Salil Vadhan.
\newblock Locating a small cluster privately.
\newblock In {\em Proceedings of the 35th ACM SIGMOD-SIGACT-SIGART Symposium on
  Principles of Database Systems}, PODS '16, pages 413--427, New York, NY, USA,
  2016. ACM.

\bibitem[Smi11]{Smith11}
Adam Smith.
\newblock Privacy-preserving statistical estimation with optimal convergence
  rates.
\newblock In {\em Proceedings of the 43rd Annual ACM Symposium on the Theory of
  Computing}, STOC '11, pages 813--822, New York, NY, USA, 2011. ACM.

\bibitem[SOAJ14]{SureshOAJ14}
Ananda~Theertha Suresh, Alon Orlitsky, Jayadev Acharya, and Ashkan Jafarpour.
\newblock Near-optimal-sample estimators for spherical {G}aussian mixtures.
\newblock In {\em Advances in Neural Information Processing Systems 27}, NIPS
  '14, pages 1395--1403. Curran Associates, Inc., 2014.

\bibitem[SSSS17]{ShokriSSS17}
Reza Shokri, Marco Stronati, Congzheng Song, and Vitaly Shmatikov.
\newblock Membership inference attacks against machine learning models.
\newblock In {\em Proceedings of the 38th IEEE Symposium on Security and
  Privacy}, SP '17, pages 3--18, Washington, DC, USA, 2017. IEEE Computer
  Society.

\bibitem[SU15]{SteinkeU15}
Thomas Steinke and Jonathan Ullman.
\newblock Interactive fingerprinting codes and the hardness of preventing false
  discovery.
\newblock In {\em Proceedings of the 28th Annual Conference on Learning
  Theory}, COLT '15, pages 1588--1628, 2015.

\bibitem[SU17a]{SteinkeU17a}
Thomas Steinke and Jonathan Ullman.
\newblock Between pure and approximate differential privacy.
\newblock {\em The Journal of Privacy and Confidentiality}, 7(2):3--22, 2017.

\bibitem[SU17b]{SteinkeU17b}
Thomas Steinke and Jonathan Ullman.
\newblock Tight lower bounds for differentially private selection.
\newblock In {\em Proceedings of the 58th Annual IEEE Symposium on Foundations
  of Computer Science}, FOCS '17, pages 552--563, Washington, DC, USA, 2017.
  IEEE Computer Society.

\bibitem[Tal94]{talagrand1994sharper}
Michel Talagrand.
\newblock Sharper bounds for gaussian and empirical processes.
\newblock {\em The Annals of Probability}, pages 28--76, 1994.

\bibitem[VC71]{VapnikC71}
Vladimir~Naumovich Vapnik and Alexey~Yakovlevich Chervonenkis.
\newblock On the uniform convergence of relative frequencies of events to their
  probabilities.
\newblock {\em Theory of Probability \& Its Applications}, 16(2):264--280,
  1971.

\bibitem[VV10]{ValiantV10a}
Gregory Valiant and Paul Valiant.
\newblock A {CLT} and tight lower bounds for estimating entropy.
\newblock {\em Electronic Colloquium on Computational Complexity (ECCC)},
  17(179), 2010.

\bibitem[Yat85]{Yatracos85}
Yannis~G. Yatracos.
\newblock Rates of convergence of minimum distance estimators and
  {K}olmogorov's entropy.
\newblock {\em The Annals of Statistics}, 13(2):768--774, 1985.

\bibitem[ZKKW20]{ZhangKKW20}
Huanyu Zhang, Gautam Kamath, Janardhan Kulkarni, and Zhiwei~Steven Wu.
\newblock Privately learning {M}arkov random fields.
\newblock In {\em Proceedings of the 37th International Conference on Machine
  Learning}, ICML '20. JMLR, Inc., 2020.

\end{thebibliography}

\appendix
\section{Useful Inequalities}
\begin{proposition}\label{prop:mappingTV}
Let $X$ and $Y$ be random variables taking values in the same set. For any function $f$,  we have $\TV(f(X),f(Y)) \leq \TV(X,Y)$.
\end{proposition}
\begin{proof}
For any set $A$ we have,
$$\mathbf{Pr}[f(X) \in A] - \mathbf{Pr}[f(Y) \in A] = \mathbf{Pr}[X \in f^{-1}(A)] - \mathbf{Pr}[Y \in f^{-1}(A)] \leq \TV(X,Y).$$
Taking the supremum of the left hand side completes the proof.
\end{proof}

\begin{corollary}\label{corollary:invertmappingTV}
Let $X$ and $Y$ be random variables taking values in the same set. For any \emph{invertible} function $f$,  we have $\TV(f(X),f(Y)) = \TV(X,Y)$.
\end{corollary}
\begin{proof}
By Proposition~\ref{prop:mappingTV}, $\TV(f(X),f(Y)) \leq \TV(X,Y)$ and $\TV(f^{-1}(f(X)),f^{-1}(f(Y))) \leq \TV(f(X),f(Y))$.
\end{proof}

\begin{proposition}\label{lem:centering-approx-gaussian}
For any $\xi \in (0,1)$, Gaussian $\mathcal{N}(\mu,\Sigma)$ and distributions $P$ that satisfies $\TV(P,\mathcal{N}(\mu,\Sigma)) \leq \xi$, the following holds. If $X_{1},X_{2} \sim P^{2}$, and $Y = (X_{1}-X_{2})/\sqrt{2} \sim Q$, then $\TV(Q,\mathcal{N}(0,\Sigma)) \leq 3\xi$.
\end{proposition}
\begin{proof}
Given $X \sim P$, the density $P$ is given by $P = \mathcal{N}(\mu,\Sigma)+\Delta$.\footnote{Recall in this paper we define distributions by their densities so  $\mathcal{N}(\mu,\Sigma)$ is the Gaussian density.} It follows from the definition of the TV distance that $\|\Delta \|_{1} \leq 2\xi$. Given a sample $X \sim P$, we let $P^{-}$ be the distribution that satisfies $-X \sim P^{-}$. The density of $P^{-}$ is given by $P^{-} = \mathcal{N}(-\mu,\Sigma)+\Delta^{-}$ where $\|\Delta^{-} \|_{1} \leq 2\xi$. The sample $Y' = X_{1} + (-X_{2})$ has density
\begin{align*}
 P_{\text{conv}} &= P\circledast P^{-} = (\mathcal{N}(\mu,\Sigma)+\Delta)\circledast (\mathcal{N}(-\mu,\Sigma)+\Delta^{-})\\
 &= \mathcal{N}(0,2\Sigma) +\mathcal{N}(\mu,\Sigma)\circledast\Delta^{-}+ \Delta\circledast \mathcal{N}(-\mu,\Sigma) + \Delta\circledast\Delta^{-}.
\end{align*}
We can now bound the TV distance between $P_{\text{conv}}$ and $\mathcal{N}(0,2\Sigma)$.
\begin{align*}
 \TV(P_{\text{conv}},\mathcal{N}(0,2\Sigma)) &= \frac{1}{2}\|P_{\text{conv}}-\mathcal{N}(0,2\Sigma) \|_{1} =\frac{1}{2} \|\mathcal{N}(0,2\Sigma) +\mathcal{N}(\mu,\Sigma)\circledast\Delta^{-}+ \Delta\circledast \mathcal{N}(-\mu,\Sigma) + \Delta\circledast\Delta^{-} - \mathcal{N}(0,2\Sigma) \|_{1}\\
 &= \frac{1}{2}\|\mathcal{N}(\mu,\Sigma)\circledast\Delta^{-}+ \Delta\circledast \mathcal{N}(-\mu,\Sigma) + \Delta\circledast\Delta^{-}) \|_{1}\\
 &\leq \frac{1}{2}\left(\|\mathcal{N}(\mu,\Sigma)\|_{1}\|\Delta^{-}\|_{1}+\|\Delta\|_{1}\|\mathcal{N}(-\mu,\Sigma)\|_{1} + \|\Delta\|_{1}\|\Delta^{-}\|_{1}\right)\\
 &\leq \frac{1}{2}\left(2\xi+4\xi^{2}\right) \leq 3\xi,
\end{align*}
where the first inequality follows from the triangle inequality together with Young's inequality. Since $\TV(P_{\text{conv}},\mathcal{N}(0,2\Sigma)) \leq 3\xi$, the statement follows immediately from Corollary~\ref{corollary:invertmappingTV} and equation~(\ref{eq:Gaussiantransform}).
\end{proof} 

\section{Omitted proofs from Section~\ref{sec:prelim}}
\subsection{Proof of Lemma~\ref{lem:pack-n-cover}}\label{sec:proof-pack-n-cover}
We first prove the inequality on the left hand side. Let $\mathcal{C}_\gamma$ be a $\gamma$-cover for $\mathcal{H}$ that has size $N(\mathcal{H},\gamma)$. If $N(\mathcal{H},\gamma) = \infty$, we are done. Otherwise, we claim there is no $2\gamma$-packing $\mathcal{P}'_{2\gamma}$ of $\mathcal{H}$ of size at least $N(\mathcal{H},\gamma) + 1$. We prove this by contradiction. Assume to the contrary that there exists a $2\gamma$-packing $\mathcal{P}'_{2\gamma}$ of $\mathcal{H}$ such that $|\mathcal{P}'_{2\gamma}| = N(\mathcal{H},\gamma) + 1$. By the pigeonhole principle, there exists $Q \in \mathcal{C}_\gamma$ and two distributions $P, P' \in \mathcal{P}'_{2\gamma}$ such that $\TV(Q, P) \leq \gamma$ and $\TV(Q, P') \leq \gamma$.  By the triangle inequality it follows that $\TV(P, P') \leq 2\gamma$ which is a contradiction, so $\mathcal{P}'_{2\gamma}$ cannot be a $2\gamma$-packing of $\mathcal{H}$. This shows that $M(\mathcal{H},2\gamma) \leq N(\mathcal{H},\gamma)$. 

We now prove the inequality on the right hand side. Let $\mathcal{P}_{\gamma}$ be a maximal $\gamma$-packing with size $M(\mathcal{H},\gamma)$. If $M(\mathcal{H},\gamma) = \infty$, we are done. Otherwise, we claim that $\mathcal{P}_{\gamma}$ is also an $\gamma$-cover of $\mathcal{H}$, and hence $N(\mathcal{H},\gamma) \leq M(\mathcal{H},\gamma)$. We can prove this by contradiction. Suppose to the contrary that there were a distribution $P \in \mathcal{H}$ with $\TV(P, \mathcal{P}_{\gamma}) > \gamma$. Then we could add $P$ to $\mathcal{P}_{\gamma}$ to produce a strictly larger packing, contradicting the maximality of $\mathcal{P}_{\gamma}$. Therefore it only remains to show that a maximal packing actually exists, which follows from a simple application of Zorn's lemma (See proof of Lemma~\ref{lem:covering-balls-implies-locally-small-cover}).

\subsection{Proof of Lemma~\ref{lem:semi-agnostic}}\label{sec:proof-lem-robust-to-agnostic}

Fix accuracy and privacy parameters $\alpha,\beta, \eps,\delta \in (0,1)$. Let $T = \lceil \log_2 (1/\alpha) \rceil$ and define sequences $\xi_{1} = \alpha/12C, \xi_{2} = 2\alpha/12C, \dots, \xi_{T+4} = 2^{T+3}\alpha/12C$. Furthermore for all $t\in[T+4]$ set $\alpha_{t} = \alpha/12$, $\beta_{t} = \beta/2(T+4)$, $\eps_{t} = \eps/2(T+4)$ and $\delta_{t} = \delta/(T+4)$. For each $t$, let $H_t$ denote the outcome of a run of the $(\xi,C)$-robust algorithm using robustness parameter $\xi_{t}$ accuracy parameters $\alpha_{t},\beta_{t}$ and privacy parameters $\eps_{t},\delta_{t}$. We then use the algorithm of Theorem~\ref{lem:private-mde} to select a hypothesis from $H_1, \dots, H_{T+4}$ using accuracy parameter $\alpha/2$, $\beta/2$ and privacy parameter $\eps/2$. By composition of DP (Lemma~\ref{lem:composition}) it follows the output of the two step procedure is $(\eps,\delta)$-DP. 

We can now argue about the accuracy of the algorithm. Given $\tilde{n}_{\mathcal{H}}^{C}\left(\frac{\alpha}{12},\frac{\beta}{ 2(T+4)},\frac{\eps}{2(T+4)},\frac{\delta}{T+4}\right)$ samples, a union bound guarantees that with probability at least $1-\beta/2$ each $H_{t}$ is accurate (assuming $\xi_{t}$ is a correct upper bound for $\mathrm{OPT}$). We condition on this event. 

The success of run $t$ of the robust PAC learner implies that $\mathrm{OPT} \in (\xi_{t-1},\xi_{t}]$, so $\TV(P,H_{t}) \leq C\xi_{t}+\alpha_{t} \leq 2C\cdot\mathrm{OPT}+\alpha/12$. Similarly, if $\mathrm{OPT} \leq \xi_{1}$ then $H_{1}$ satisfies $\TV(P,H_{1})\leq C\xi_{1} + \alpha_{1} = \alpha/12 + \alpha/12 = \alpha/6$. Finally, if $\mathrm{OPT} > \xi_{T+4}$ this implies $\mathrm{OPT} > 8/12C$, so \emph{any} distribution $H'$ satisfies $\TV(P,H') \leq 2C\cdot\mathrm{OPT}$ trivially. So regardless of $\mathrm{OPT}$, there is a run $t$ that satisfies $\TV(P,H_{t}) \leq 2C\cdot\mathrm{OPT} + \alpha/6$. Finally, Theorem~\ref{lem:private-mde} guarantees that, with probability greater than $1-\beta/2$, the second step in our above procedure will output a distribution $\widehat{H}$ such that $\TV(P,\widehat{H}) \leq 3(2C\cdot\mathrm{OPT} + \alpha/6)+\alpha/2 = 6C\cdot\mathrm{OPT} + \alpha$, as long as $n=O\left(\frac{\log(T/\beta)}{\alpha^{2}}+\frac{\log(T/\beta)}{\alpha\eps}\right)$. A union bound together with setting $$ n =\tilde{n}_{\mathcal{H}}^{C}\left(\frac{\alpha}{12},\frac{\beta}{ 2(T+4)},\frac{\eps}{2(T+4)},\frac{\delta}{T+4}\right) + O\left(\frac{\log(T/\beta)}{\alpha^2} +\frac{\log(T/\beta)}{\alpha\eps}\right)$$ completes the proof.

\end{document}